\documentclass{article} 
\usepackage{fullpage,times}
\usepackage{parskip}

\usepackage{user}

\title{\textbf{Decoder-Only or Encoder-Decoder? Interpreting Language Model as a Regularized Encoder-Decoder}}

\author{
    \textbf{Zihao Fu\textsuperscript{\rm 1}, 
    Wai Lam\textsuperscript{\rm 2}, Qian Yu\textsuperscript{\rm 3}, Anthony Man-Cho So\textsuperscript{\rm 2}}\\ \textbf{Shengding Hu\textsuperscript{\rm 4}, Zhiyuan Liu\textsuperscript{\rm 4},
    Nigel Collier\textsuperscript{\rm 1}}
}

\date{
    \textsuperscript{\rm 1}Language Technology Lab, University of Cambridge,\\
    \textsuperscript{\rm 2}The Chinese University of Hong Kong, 
    \textsuperscript{\rm 3}JD.com\\
    \textsuperscript{\rm 4}Department of Computer Science and Technology, Tsinghua University\\
    \{zf268,nhc30\}@cam.ac.uk,
    \{wlam,manchoso\}@se.cuhk.edu.hk\\
    \{hsd20, liuzy\}@mails.tsinghua.edu.cn, yuqian81@jd.com
}

\begin{document}
%\pagecolor{almond}

\maketitle

\begin{abstract}
  The sequence-to-sequence (seq2seq) task aims at generating the target sequence based on the given input source sequence. Traditionally, most of the seq2seq task is resolved by the Encoder-Decoder framework which requires an encoder to encode the source sequence and a decoder to generate the target text. Recently, a bunch of new approaches has emerged that apply decoder-only language models directly to the seq2seq task. Despite the significant advancements in applying language models to the seq2seq task, there is still a lack of thorough analysis on the effectiveness of the decoder-only language model architecture. This paper aims to address this gap by conducting a detailed comparison between the encoder-decoder architecture and the decoder-only language model framework through the analysis of a regularized encoder-decoder structure. This structure is designed to replicate all behaviors in the classical decoder-only language model but has an encoder and a decoder making it easier to be compared with the classical encoder-decoder structure. Based on the analysis, we unveil the attention degeneration problem in the language model, namely, as the generation step number grows, less and less attention is focused on the source sequence. To give a quantitative understanding of this problem, we conduct a theoretical sensitivity analysis of the attention output with respect to the source input. Grounded on our analysis, we propose a novel partial attention language model to solve the attention degeneration problem. Experimental results on machine translation, summarization, and data-to-text generation tasks support our analysis and demonstrate the effectiveness of our proposed model. 
\end{abstract}

\section{Introduction}
The sequence-to-sequence (seq2seq) task \cite{sutskever2011generating,sutskever2014sequence,cho2014learning,opennmt,ott2019fairseq} sees rapid growth during the past few years. It takes a source sequence as input and generates a corresponding target sequence. Several natural language generation tasks, such as translation, summarization, data-to-text generation, and story generation, naturally fall under the category of seq2seq tasks. Moreover, recently, several other non-generation tasks, including question answering, classification, and etc. \cite{raffel2020exploring}, have also been unified under the seq2seq paradigm. Traditionally, most of the existing seq2seq frameworks have employed the Encoder-Decoder (ED) architecture \cite{cho2014learning,sutskever2014sequence}, where an encoder is responsible for encoding the input data into a hidden space, while a decoder is used to generate the target output text.

\begin{figure}[t]
\centering
\includegraphics[width=0.4\columnwidth]{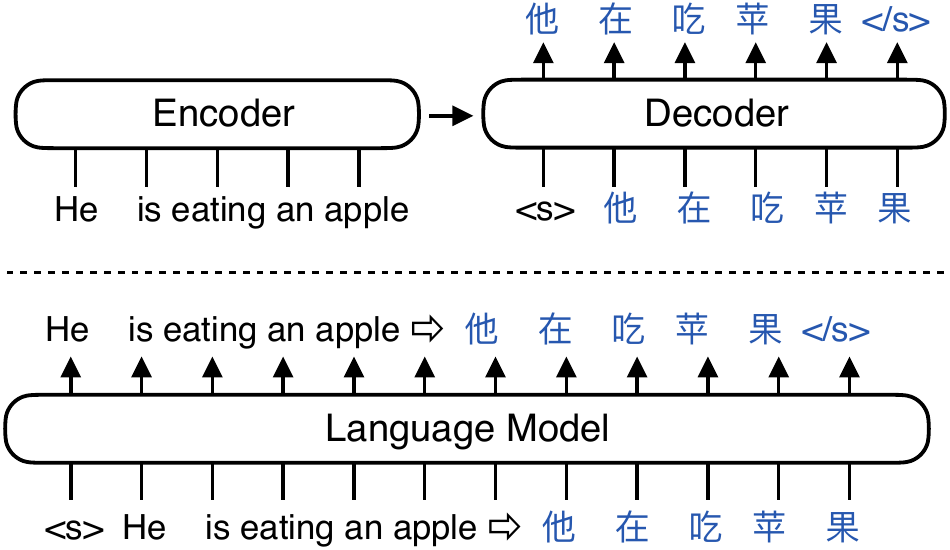}
\caption{Encoder-Decoder (ED) framework and decoder-only Language Model (LM).}
\label{fig:problem}
\end{figure}

Recently, many promising large language models (GPT \cite{radford2018improving}, GPT2 \cite{radford2019language}, GPT3 \cite{brown2020language}, InstructGPT/ChatGPT \cite{ouyang2022training}, Palm \cite{chowdhery2022palm}, OPT \cite{zhang2022opt}, Bloom \cite{scao2022bloom}, Galactica \cite{taylor2022galactica}, Llama \cite{touvron2023llama}) have emerged that directly employ a decoder-only Language Model (LM) \cite{bengio2003neural,mikolov2010recurrent,mikolov2011extensions,mikolov2012context,vig2019analyzing} to solve the seq2seq tasks (\Cref{fig:problem}). It concatenates each source sequence with the corresponding target sequence as a single one and trains the LM on this new collection of sequences. The decoder-only structure \cite{liu2018generating,radford2019language,
dong2019unified,brown2020language,ghosh2017affect} has many obvious advantages over the commonly used ED framework. First, it only has a decoder and thus reduces the model size significantly. Second, LM can be pre-trained on unlabeled text data which is much easier to obtain. Moreover, LM has many good properties including parameter sharing, layer-wise coordination, etc. %\citet{liu2018generating} propose to use a prefix LM to generate the Wikipedia sentences. Other examples including GPT models \cite{radford2019language,brown2020language} and UNILM \cite{dong2019unified} try to pre-train an LM and apply it on specific seq2seq tasks with finetuning.

Despite the remarkable achievements of recent large language models, it is still unclear whether applying the LM in the seq2seq task is a promising choice regarding the performance of this task. \citet{liu2018generating} show that the LM gets some gains over the ED structure in the summarization task while \citet{radford2019language,brown2020language} show that the pre-trained LMs can be applied in the unsupervised translation task. On the other hand, \citet{raffel2020exploring} indicate that the LMs, when directly applied in the machine translation tasks, perform worse than the classical ED structure. \citet{deng2023recent} show that ED structure still performs better than decoder-only LM structure. \citet{zhu2020incorporating} indicate that when applying the pre-trained LM to a language domain beyond the corpus for training, no signiﬁcant improvement is observed. \citet{fu2021theoretical} show that LM suffers from the repetition generation problem. Therefore, it motivates us to investigate the deployment and the performance of LMs in seq2seq tasks.

We propose to study the characteristics of LM when applied in the seq2seq task by conducting a detailed contrastive study between the LM and the ED structure. Grounded with our analysis, we unveil the attention degeneration problem and propose a novel partial attention language model to solve it. Specifically, we first propose to analyze a variant of the traditional ED structure named as \textbf{Regularized Encoder-Decoder} (RED) framework. It is designed to replicate all behaviors in classical LM while structured with an encoder and a decoder. This structure facilitates the comparison with the ED structure. By comparing with the RED framework and the ED framework, we find that some parts of the RED structure benefit the seq2seq task while some parts do not. 
Moreover, we find the defects of the LM applied in a seq2seq task are partially caused by the \textbf{Attention Degeneration Problem} (ADP) in its attention component. We conduct a theoretical analysis of this problem by deriving an upper bound of the attention sensitivity and we find that the sensitivity decreases as the length of the generated sequence grows in the LM. 
Based on this analysis, we propose a novel \textbf{Partial Attention Language Model} (PALM) which takes advantage of some LM components while avoiding other unfavorable ones. Our experimental results on machine translation, summarization, and data-to-text generation datasets support the correctness of our analysis and demonstrate the effectiveness of our proposed PALM structure.

Our contributions can be summarized as follows: (1) We conduct a detailed analysis of applying the LM in the seq2seq task. (2) We conduct a theoretical analysis of the attention degeneration problem and carry out extensive experiments to support it. (3) Based on our analysis, we propose PALM to alleviate the weakness of LM, from which the recent LM-based methods in seq2seq tasks can benefit.

\section{Related Works}

Numerous works have recently proposed using LM in the seq2seq task. \citet{liu2018generating} suggest using decoder-only LM to generate summaries of Wikipedia articles, whereas \citet{raffel2020exploring} demonstrate through extensive experiments that LMs perform worse than ED frameworks in machine translation tasks. On the other hand, several works propose using pre-training to enhance performance. For instance, \citet{radford2019language,brown2020language} propose pre-training the GPT LM and fine-tuning it on an unsupervised seq2seq task. UniLM \cite{dong2019unified} adopts the prefix LM \cite{raffel2020exploring}, which uses a fully visible mask of the source sequence instead of the traditional causal mask. Additionally, \citet{conneau2019cross,conneau2020unsupervised} propose the XLM model pre-trained on a monolingual dataset. However, \citet{zhu2020incorporating} show that when applying pre-trained LMs to a language domain beyond the training corpus, there is no significant improvement. Nevertheless, these results heavily rely on the pre-training process, and it remains unclear whether the LM structure itself is suitable for handling the seq2seq task.

Compared to the ED structure, the LM has numerous advantageous features that have been proven to enhance its performance. First, it utilizes the parameter sharing technique \cite{dehghani2018universal,xia2019tied,lan2019albert,conneau2019cross} to share parameters for networks that handle both the source and target sequences. This reduces the model size and improves performance through parameter sharing. Additionally, the LM incorporates a layer-wise coordination mechanism \cite{belinkov2017neural,he2018layer,peters2018deep} that enables the decoder to attend to each corresponding encoder layer's output, allowing it to access various levels of source information. Finally, \citet{dong2021attention} demonstrate that the attention matrix for the decoder-only is a full rank matrix since it is a triangular matrix. This is superior to the encoder-decoder attention matrix, which may not be full rank.

\section{Contrastive Study of Language Model}\label{sec:cslm}

In this section, we conduct a contrastive study of the Language Model (LM) and the traditional Encoder-Decoder (ED) structure. We propose to analyze a Regularized Encoder-Decoder (RED) framework which is designed to replicate all behaviors in the classical LM but with a structure of an encoder and a decoder. We give a detailed explanation of how each component in the RED framework imitates the behaviors in the LM and how it is different from the ED framework. We also illustrate how these components benefit or harm the overall performance. Afterwards, we analyze the attention degeneration problem and propose a theoretical analysis to unveil the cause.

\subsection{Preliminary}\label{sec:pre}
Our analysis is based on the prevalent Transformer architecture \cite{vaswani2017attention,ott2019fairseq}. As illustrated in \Cref{fig:compare}, we show the main components of the models and the detailed model structure can be found in \citet{ott2019fairseq}. In a seq2seq task, we denote the input source sequence as $s=[s_1,s_2,\cdots,s_{|s|}]$ and the target sequence as $t=[t_1,t_2,\cdots,t_{|t|}]$, where $|\cdot|$ is the length of the sequence and $s_i$ is the $i$th source word. The positional token for $s$ and $t$ is denoted as $p_s=[1,2,\cdots,|s|]$ and $p_t=[1,2,\cdots,|t|]$. The attention layer is denoted as $\mathtt{ATT}(Q,K,V)$, and $Z=\mathtt{ATT}(Q,K,V)=\mathtt{Softmax}(QW_QW_K^\top K^\top /\sqrt{d})VW_V=\mathtt{Softmax}(QA^\top K^\top /\sqrt{d})VW_V$,
%\begin{equation}\small
%  $\begin{aligned}$
%\end{aligned}\end{equation}
in which $A^\top =W_QW_K^\top $ and $W_Q,W_K,W_V\in \mathbb{R}^{d\times d}$ are trainable parameters that transform input matrices into another space. $Q\in \mathbb{R}^{d_Q\times d},K\in \mathbb{R}^{d_K\times d},V\in \mathbb{R}^{d_V\times d}$ stand for the query, key and value matrices,  $d_Q$, $d_K$, and $d_V$ are the corresponding dimensions of the matrices while $Z\in \mathbb{R}^{d_Q\times d}$ is the output of the attention layer. Here, the $\mathtt{Softmax}$ operation is applied to each row of the matrix concerned. 

As shown on the left of \Cref{fig:compare}, ED's encoder contains multiple Transformer encoder blocks denoted with the shaded rectangle. $G_l^E\in \mathbb{R}^{|s|\times d}$ is the input feature matrix for the self attention layer $\mathtt{ATT}_l^E$ while $H_l^E\in \mathbb{R}^{|s|\times d}$ is the output matrix of the $l$th encoder block. $G_l^E$ equals to the sum of the word embedding and the positional embedding for the first block ($l=1$) and $G_l^E$ equals to $H_{l-1}^E$ when $l>1$. In ED's decoder, we denote $G_l^D\in \mathbb{R}^{|t|\times d}$ as the output of the self attention layer $\mathtt{ATT}_l^D$. We feed $G_l^D$ into the encoder attention $\mathtt{ATT}_l^J$ and denote the output as $Q_l^D\in \mathbb{R}^{|t|\times d}$. $H_l^D\in \mathbb{R}^{|t|\times d}$ is the output matrix of the $l$th decoder block. We denote $\mathcal{L}^D$ as the negative log likelihood loss for the target sequence. Similarly, as shown on the right of \Cref{fig:compare}, the RED structure has the same definition of the matrices as that in ED. Moreover, We denote $\mathcal{L}^E$ as the negative log likelihood loss for the source sequence and we denote the unidirectional cross attention as $\mathtt{ATT}_l$. For more structural details we refer the readers to Appendix \ref{sec:structure} and the original papers \cite{vaswani2017attention,ott2019fairseq}.

\begin{figure*}[!t]
  \centering
  \includegraphics[width=1.0\columnwidth]{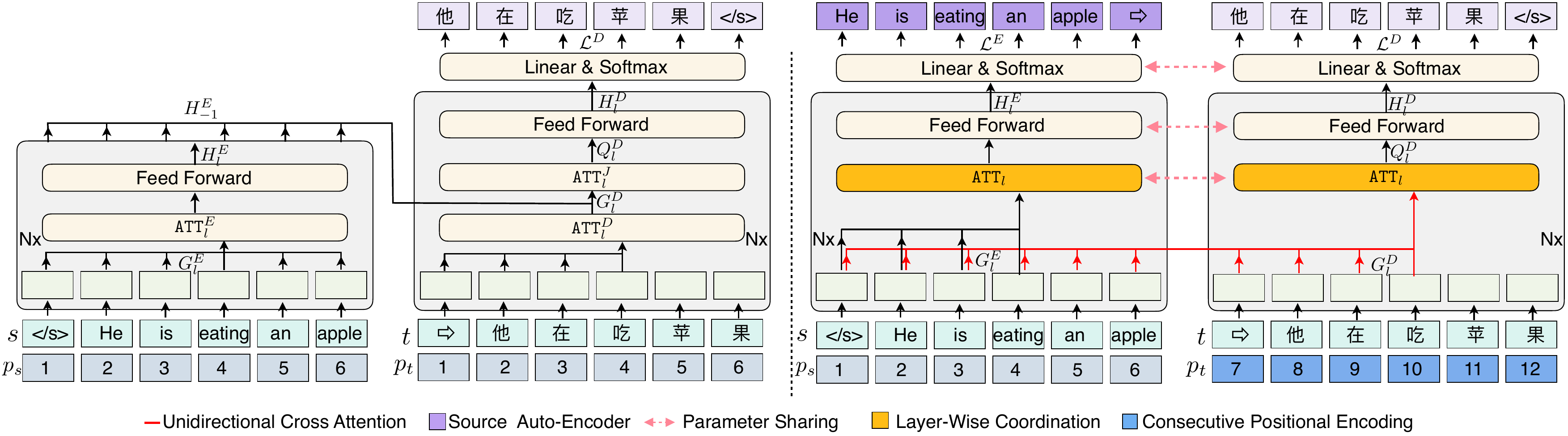}
  \caption{Encoder-Decoder framework (left) and Regularized Encoder-Decoder framework (right). }
  \label{fig:compare}
\end{figure*}

\subsection{Regularized Encoder-Decoder}\label{sec:red}

Though the decoder-only Language Model (LM) is simply a decoder, it is still difficult to be compared with an Encoder-Decoder (ED) structure because this decoder handles both the source sequence and the target sequence together. To facilitate the comparison between the ED and LM structure, we propose to analyze a Regularized Encoder-Decoder (RED) framework as illustrated in \Cref{fig:compare}. It is a variant of the traditional ED framework while replicating the behaviors of an LM. Compared with the traditional ED structure, the RED framework mainly has the following different components: 
An unidirectional cross attention attends to both the source matrix and the target matrix simultaneously; 
a source auto-encoder recovers the input source; 
a parameter sharing mechanism shares the parameters between the encoder and the decoder; 
a layer-wise coordination component makes each decoder layer attending to the corresponding encoder layer output;
a consecutive positional encoding utilizes a positional encoding starting from the length of the source tokens in the decoder.

\textbf{Unidirectional Cross Attention.}
The main difference between the ED framework and the LM is how the input source information is merged into the decoder. 
As illustrated in \Cref{fig:compare}, the ED framework first uses multiple Transformer blocks to extract features $H_{-1}^E$ from the source sequence $s$. Afterwards, it utilizes a self attention $\mathtt{ATT}_l^D$ to get the feature matrix $G_l^D$. It then uses an encoder attention $\mathtt{ATT}_l^J$ to take $G_l^D$ as query and uses the encoder's final output $H_{-1}^E$ as the key and value to calculate $Q_l^D$.
On the other hand, an LM uses an unidirectional attention to handle the concatenated features. To simulate this mechanism in the LM, as illustrated in \Cref{fig:compare}, the RED framework uses unidirectional cross attention $\mathtt{ATT}_l$ which attends to both the source matrix $G_l^E$ and the target matrix $G_l^D$ simultaneously. Since it attends to all features with one attention, the output matrix $Q_l^D$ of the attention layer becomes less sensitive to the input source matrix $G_l^E$ especially when it has already generated many words and $G_l^D$ becomes relatively long. We call this the attention degeneration problem and we will analyze it in detail in Section \ref{sec:attention-degeneration-problem}.

\textbf{Source Auto-Encoder.} The traditional ED structure only predicts the probability for the target sequence and just takes the source sequence as features. On the contrary, an LM predicts the probability for the whole concatenated sequence including the source sequence. Therefore, in the RED framework, we adopt a Source Auto-Encoder (SAE) component to realize this mechanism. As shown in \Cref{fig:compare}, the overall loss is composed of the decoder loss $\mathcal{L}^D$ and the SAE loss $\mathcal{L}^E$. $\mathcal{L}^D$ is the same as that in the ED framework while $\mathcal{L}^E$ is actually a regularizer that recovers the source sequence $s$ to itself. Therefore, it can alleviate the overfitting problem in training and thus improve the performance.

\textbf{Parameter Sharing.} In the traditional ED framework, the encoder and the decoder have their own parameters. On the contrary, in an LM, the source and target sequences are concatenated and passed through the same network with the same parameters.
To simulate this property, the RED framework shares the parameters between the encoder and the decoder. Parameter sharing techniques \cite{dehghani2018universal,xia2019tied,lan2019albert,conneau2019cross} can be recognized as another regularizer that prevents the model from having too many parameters and thus alleviates the overfitting problem.

\textbf{Layer-Wise Coordination.} The traditional ED structure feeds the source sequence into the encoder and get the final source hidden representation matrix $H_{-1}^E$. Then, each decoder layer takes $H_{-1}^E$ as the input matrix for the encoder attention layer $\mathtt{ATT}_l^J$. On the contrary, if the source and target sequence are concatenated and feed into the LM, the attention component in each layer will take the current layer's hidden representation instead of using the same representation matrix. To imitate this feature, the RED framework adapts the layer-wise coordination component which uses attention $\mathtt{ATT}_l$ in the decoder to attend to each corresponding encoder layer feature $G_l^E$. This method enables the decoder to access multiple levels of source information and thus improves the performance \cite{belinkov2017neural,he2018layer,peters2018deep}.

\textbf{Consecutive Positional Encoding.} In the traditional ED structure, the source sequence $s$ and the target sequence $t$ have their own positional encoding $p_s$ and $p_t$ starting both from 1 to the length of the source and target sequences. However, if the source and target sequences are concatenated and sent into the LM, the positional encoding for the target sequence will not start from 1. To realize this mechanism, the RED framework uses a Consecutive Positional Encoding (CPE) which encodes the position $p_t$ for the target sequence $t$ starting from $|s|+1$ instead of restarting from 1. It should be noted that this mechanism has been proved in many research works \cite{he2018layer,conneau2019cross} to be less effective compared with its counterpart called Separate Positional Encoding (SPE) which restarts the positional embedding for the target sentence. This is because restarting the positional embedding makes the model more aware of different languages. Moreover, we observe that the LM always has strong attention on the first element of the target sequence. This token may also be used by the model to differentiate the input source and output target. Using a fixed starting position for the target sequence makes it easier for the model to find it.

\subsection{Attention Degeneration Problem}\label{sec:attention-degeneration-problem}

Since the unidirectional cross attention imposes attention on both the source sequence and the target sequence simultaneously, less and less attention will be focused on the source sequence as the target sequence length grows. This is the attention degeneration problem.
To quantitatively understand why the ED structure does not have this problem and how the influence of the source sequence decreases in LM, we propose a theoretical analysis on the sensitivity of $\mathtt{ATT}_l$ and $\mathtt{ATT}_l^J$ in the RED framework and the ED framework where the source information is merged into the decoder. The sensitivity measures the influence of the source vectors to the attention layer's output vectors. To simplify the analysis, we only analyze one head of the attention. Since we focus on analyzing the effect of concatenating the target sequence, we assume $\mathtt{ATT}_l$ and $\mathtt{ATT}_l^J$ both use the unidirectional encoding and have the same parameters to make them easier to be compared with each other. 

We denote the input source matrix for the attention layer as $X=[x_1,x_2,\cdots,x_N]^\top \in \mathbb{R}^{N\times d}$, where $x_i\in \mathbb{R}^{d}$ and $d$ is the vector dimension, $N$ is the length of the source sequence. We denote the target matrix as $Y=[y_1,y_2,\cdots,y_i]^\top \in \mathbb{R}^{i \times d}$, where $i$ is the current step. We denote the encoder attention as $Z^{E}=\mathtt{ATT}(Y, X, X)$ where $X$, $Y$, $Z^E$, and $\mathtt{ATT}$ correspond to $H_{-1}^E$, $G_l^D$, $Q_l^D$, and $\mathtt{ATT}_l^J$ in the ED structure respectively (\Cref{fig:compare}). On the other hand, the unidirectional cross attention can be denoted as $Z^{C}=\mathtt{ATT}(Y, [X^\top ,Y^\top ]^\top , [X^\top ,Y^\top ]^\top )$, where $X$, $Y$, $Z$, and $\mathtt{ATT}$ correspond to $G_l^E$, $G_l^D$, $Q_l^D$ and $\mathtt{ATT}_l$ in the RED structure respectively (\Cref{fig:compare}).

We begin our analysis by defining a notion of sensitivity. Intuitively, given a function $y=f(x)$, the sensitivity of $x$ on $y$ can be described as how the output vector $y$ changes ($\Delta y$) when imposing a pertubation $\Delta x$ on the input vector $x$. However, imposing different $\Delta x$ leads to different $\Delta y$, we propose to study the upper bound on the ratio between the magnitude of $x$ and $y$ based on the following proposition:

\begin{proposition}
  \small
  \label{prop:pertube}
  Given a function $y=f(x)$ with a Jacobian matrix $J_f$, if we have a pertubation vector $\Delta x$ and $y+\Delta y=f(x + \Delta x)$, then\vspace{-1em}
  \begin{equation}\small \frac{\|\Delta y\|}{\|\Delta x\|} \le \|J_f\| + o(1) .
  \end{equation}
\end{proposition}

The proof is provided in Appendix \ref{proof:pertube}.  We can observe that the ratio is upper bounded with the norm of the Jacobian matrix $J_f$ plus $o(1)$.  Then, we focus on $J_f$ and define the sensitivity of the attention component as the norm of the Jacobian matrix of the $i$th row vector $z_i$ of the attention output $Z$ with respect to the $j$th row vector $x_j$ in $X$.

\begin{definition}
  [Sensitivity] Let the vectors $x_j$ and $z_i$ be defined above. The sensitivity of $z_i$ with respect to $x_j$ is defined as the norm of the Jacobian matrix:
  \begin{equation}\small S_{ij}=\|J_{ij}\|=\|\frac{\partial z_i}{\partial x_j}\|.  \end{equation}
\end{definition}

Intuitively, if the sensitivity is high, the output vector is closely related to the input source vector. Otherwise, the output vector may not contain enough information about the input source vector and the output is likely to be generated by a simple LM. From the Section \ref{sec:pre}, $Z=\mathtt{ATT}(Y,X,X)=\mathtt{Softmax}(YA^\top X^\top/\sqrt{d})XW_V=PXW_V$, in which $P\in \mathbb{R}^{d_Y \times d_X}$ and $d_X,d_Y$ are the first dimensions of $X$ and $Y$ respectively. Following \citet{kim2020lipschitz}, the Softmax matrix $P$ is a stochastic matrix, namely, its entries are non-negative and its rows sum to 1. For each element $p_{ij}$ in $P$, $p_{ij}\in[0,1]$ and they have an equal chance of receiving attention. Therefore, we have $\mathbb{E}(p_{ij})=\frac{1}{d_X}$. To compare the sensitivity between the two kinds of attention, with the above assumption, we have the following theorem:

\begin{theorem}
  \label{thm:jacobs}
  For $Z^E=\mathtt{ATT}(Y,X,X)$, where $\|X\|, \|Y\|, \|A\|, \|W_V\|$ are bounded, $\exists$ $C_3\ge 0,\delta \in(0,1)$, with probability at least $1-\delta^2$,
  \begin{equation}\small
    \|J_{ij}^E\|\le C_3 (\frac{1}{N}+\sqrt{\ln \frac{1}{\delta}}).\end{equation}\vspace{-1em}

For $Z^C=\mathtt{ATT}(Y,[X^\top ,Y^\top ],[X^\top ,Y^\top ])$, with probability at least $1-\delta^2$,
\begin{equation}\small
  \|J_{ij}^C\|\le  C_3 (\frac{1}{N+i}+\sqrt{\ln \frac{1}{\delta}}).  \end{equation}\vspace{-1em}
\end{theorem}

The proof is provided in Appendix \ref{proof:proof-thm-bound}. It can be directly observed that the upper bound of the Jacobian matrix norm is negatively correlated with the value sequence length. Therefore, $\|J_{ij}^{C}\|$ has a smaller upper bound than $\|J_{ij}^{E}\|$, which implies that the encoder attention ($\mathtt{ATT}_l^J$) output is more likely to be sensitive to the input source vectors. As a result, if the input vector changes, the output of the encoder attention changes more significantly than that of the cross unidirectional attention. Moreover, the difference between the two upper bounds becomes even larger as the generating step number $i$ grows. It shows that the sensitivity to the source input decreases in the unidirectional cross attention as the step number $i$ grows. We will conduct extensive experiments to verify these observations.

\section{Partial Attention Language Model}

\begin{figure}[!t]
  \centering
  \includegraphics[width=0.5\columnwidth]{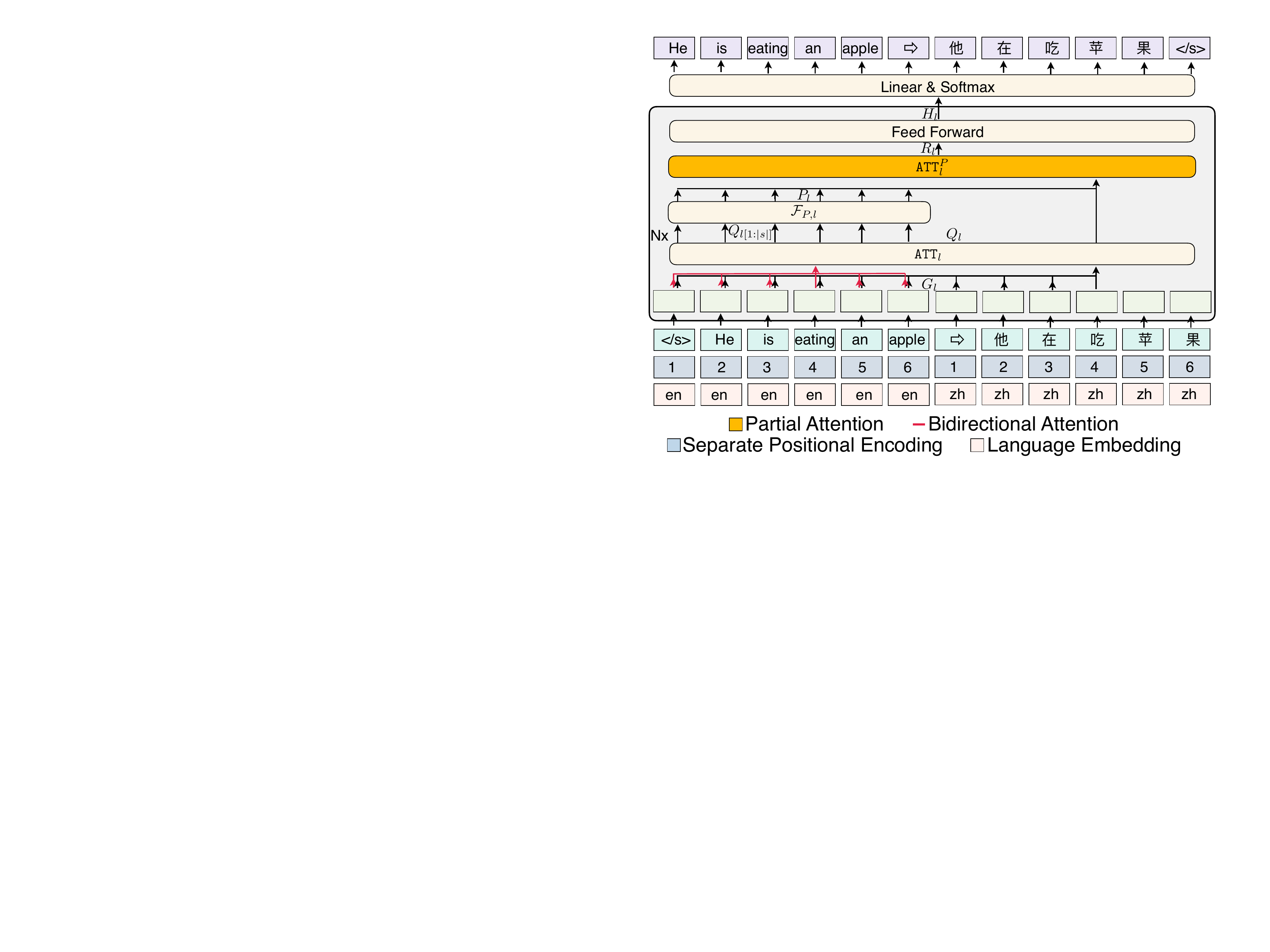}
  \caption{PALM framework. }
  \label{fig:palm}
  \vspace{-1em}
\end{figure}

To overcome the defects in the LM discussed above, we propose a Partial Attention Language Model (PALM) which is shown in \Cref{fig:palm}. We keep the effective components and remove the unfavorable components as analyzed in the RED framework. First, we propose a novel partial attention component to alleviate the attention degeneration problem. Besides, we designate a few adjustments according to the aforementioned analysis of defects in the LM. Specifically, we propose to use a separate positional embedding to replace the consecutive positional embedding and use bidirectional attention for the source sequence. Moreover, we add a language encoding layer and keep other components unchanged.

\textbf{Partial Attention Component.} To alleviate the attention degeneration problem, we propose the Partial Attention (PA) component. Similar to the ED framework, it adds a new attention layer $\mathtt{ATT}_l^P$ that only focuses on the source part of the feature matrix $Q_l$ which is denoted as $Q_{l[1:|s|]}\in \mathbb{R}^{|s|\times d}$ and ${[1:|s|]}$ stands for taking the first to the $|s|$th row vectors of $Q_l$. Since $Q_{l[1:|s|]}$ is not affected as the target sequence grows, the attention degeneration problem can thus be alleviated. As shown in \Cref{fig:palm}, $Q_{l[1:|s|]}$ is passed through a consecutive feedforward layer which is denoted as $P_l=\mathcal{F}_{P,l}(Q_{l[1:|s|]})$. 
Afterwards, we use another attention layer $\mathtt{ATT}_l^P$ to attend to each vector of $P_l$ by each vector in $G_l$ as
$R_l=\mathtt{ATT}_l^P(Q_l,P_l,P_l)$.
Then, $R_l$ is passed through a feedforward layer to get the hidden representation matrix $H_l$ for the next layer. Since $\mathtt{ATT}_l^P$ takes $P_l$ as the key and value matrix which has a fixed length and does not contain any target part, it can be concluded from Theorem \ref{thm:jacobs} that the sensitivity upper bound will not change as the generation step grows. Therefore, it can alleviate the attention degeneration problem.

\textbf{Bidirectional Attention.} In $\mathtt{ATT}_l$, we propose to set a bidirectional attention mask \cite{dong2019unified,raffel2020exploring} for $G_l[1:|s|]$ making each position in the source sequence aware of the whole source sequence.

\textbf{Separate Positional Encoding.} As discussed in Section \ref{sec:red}, we adopt the separate positional encoding in PALM to help the model differentiate between the source and target parts.

\textbf{Language Embedding.} We further adopt the language embedding \cite{conneau2019cross,conneau2020unsupervised} to help the model differentiate between the source and target sequences.

\section{Experiments}

\subsection{Experimental Settings}
We conduct experiments on several tasks to show the effectiveness of our proposed method. All the experiments are conducted with a single NVIDIA(R) TITAN RTX graphics card with Intel(R) Xeon(R) Silver 4210 CPU @ 2.20GHz CPU. All the datasets used in this paper are very commonly used in corresponding tasks. The detailed license and how the datasets are made can be found in the corresponding instructions. To the best of our knowledge, no personally identifiable information or offensive content have been reported in these datasets.

\textbf{Machine Translation} task \cite{cho2014learning,bahdanau2014neural,luong2015effective,wu2016google,gehring2017convolutional,johnson2017google,koehn2017six,vaswani2018tensor2tensor} is the most popular NLP task and we adopt several language pairs in the commonly used IWSLT’14 dataset \cite{cettolo2014report} (Creative Commons License) including De-En, En-De, It-En, En-It, En-Fr, Es-En, En-Es, Ru-En, En-Ru, He-En, En-He, Ro-En, En-Ro. Both the source and target sentences are encoded by the byte-pair encoding \cite{sennrich2016controlling} method with 10,000 subword units in a shared dictionary. In the ED model, we adopt the traditional Transformer model \cite{vaswani2017attention,ott2019fairseq} and keep all hyperparameters setting the same as the ``transformer\_iwslt\_de\_en'' architecture in fairseq \cite{ott2019fairseq}. To make a fair comparison between the models, we also keep all the transformer layer number to be 6 in all models. We use the BLEU score \cite{papineni2002bleu} as the evaluation metric.

\textbf{Data-to-Text} task can also be accomplished with Transformer. We investigate the performance on three datasets, namely the WebNLG \cite{perez2016building,gardent2017creating,gardent2017webnlg,ferreira2018enriching,shimorina2019creating} dataset (CC BY-NC-SA 4.0 License), E2E \cite{novikova2017e2e,duvsek2019semantic} dataset, and the WITA \cite{fu2020partially} dataset. The WebNLG is a human-annotated dataset aiming at generating text describing given knowledge base triples. The E2E is made with the Crowd Flower platform to generate restaurant comments based on their properties. The WITA dataset is an automatically generated dataset that generates Wikipedia article sentences based on input knowledge triples. We use the unannotated target sentences as the reference. We keep all the model hyperparameters the same as that in the machine translation task and evaluate the overall metrics with several evaluation metrics including BLEU \cite{papineni2002bleu}, ROUGE$_L$ \cite{lin2004rouge}, METEOR \cite{banerjee2005meteor}, NIST \cite{doddington2002automatic}, and CIDEr \cite{vedantam2015cider}.

\textbf{Summarization} task \cite{rush2015neural,chopra2016abstractive,nallapati2016abstractive,cheng2016neural,see2017get,paulus2018deep,pilault2020extractive} can also be implemented with a Transformer architecture. We conduct the experiments on XSUM \cite{narayan2018don} dataset (MIT License). We take the articles as the source sequence while using the summarization as the target sequence. We keep all the model hyperparameters the same as that in the machine translation task and report the ROUGE$_1$, ROUGE$_2$, ROUGE$_L$ metrics.

\subsection{Comparison Models}
We compare the PALM structure against the following models.

\textbf{ED} is the traditional Encoder-Decoder framework built on Transformer \cite{vaswani2017attention} implemented with fairseq~\cite{ott2019fairseq}.

\textbf{LM} is the traditional Language Model built with a Transformer decoder. It concatenates the source sequence and the target sequence to train an LM. When testing, it predicts the target sequence after inputting the source sequence into the LM.% The model structure is the same as that of the ED model except that LM does not have the encoder.

\textbf{LM-SPE} adopts the Separate Positional Encoding \cite{conneau2019cross,he2018layer} in an LM as the positional encoding. The first position for the target sequence is always set to 1.

\textbf{LM-LE} uses a Language Embedding \cite{conneau2019cross,conneau2020unsupervised} to help LM differentiate source and target sequence.

\textbf{LM-PA} adopts our proposed Partial Attention component in an LM. Different from PALM, it uses unidirectional attention and does not use the LE and SPE components.

\textbf{PreLM} uses a Prefix Language Model \cite{liu2018generating,dong2019unified,raffel2020exploring} to generate target sequences. Different from LM, it uses a fully-visible masking of the input source instead of using a unidirectional causal mask. The mask enables each source word depending on the whole input sequence instead of the previous one.

\subsection{Experimental Results}

\begin{table*}[t]
    \centering
    \small
    %\setlength{\leftskip}{-10pt}
    
    %\begin{tabular}{l@{\hskip 3mm}|c@{\hskip 3mm}c@{\hskip 3mm}c@{\hskip 3mm}c@{\hskip 3mm}c@{\hskip 3mm}c@{\hskip 3mm}c@{\hskip 3mm}c@{\hskip 3mm}c@{\hskip 3mm}c@{\hskip 3mm}c@{\hskip 3mm}c@{\hskip 3mm}c@{\hskip 3mm}|c@{\hskip 3mm}c} 
  \resizebox{1.\textwidth}{!}{
      \begin{tabular}{l@{\hskip 2mm}|c@{\hskip 2mm}c@{\hskip 2mm}c@{\hskip 2mm}c@{\hskip 2mm}c@{\hskip 2mm}c@{\hskip 2mm}c@{\hskip 2mm}c@{\hskip 2mm}c@{\hskip 2mm}c@{\hskip 2mm}c@{\hskip 2mm}c@{\hskip 2mm}c@{\hskip 2mm}|c@{\hskip 2mm}c}
      \toprule
      {} &  De-En &  En-De &  It-En &  En-It &  En-Fr &  Es-En &  En-Es &  Ru-En &  En-Ru &  He-En &  En-He &  Ro-En &  En-Ro & Avg. &   \#Paras \\
      \hline
      ED           &  34.18 &  28.00 &  31.96 &  29.42 &  40.86 &  40.99 &  37.53 &  23.09 &  18.20 &  38.09 &  25.41 &  38.14 &  28.30 &   31.86 &       47.1M \\
  LM           &  33.19 &  26.43 &  30.92 &  28.64 &  39.16 &  39.33 &  36.67 &  22.25 &  17.53 &  34.81 &  24.35 &  35.51 &  27.02 &   30.45 &       29.3M \\
  LM-SPE       &  33.35 &  27.35 &  31.36 &  28.87 &  39.93 &  39.69 &  36.99 &  21.89 &  17.98 &  34.89 &  24.80 &  35.19 &  27.72 &   30.77 &       29.3M \\
  LM-LE        &  33.58 &  27.46 &  31.38 &  29.03 &  40.14 &  39.87 &  37.05 &  22.24 &  18.08 &  34.80 &  24.31 &  35.76 &  27.96 &   30.90 &       29.3M \\
  LM-PA        &  34.54 &  28.35 &  32.01 &  29.70 &  40.15 &  40.58 &  37.24 &  22.98 &  18.45 &  35.64 &  25.37 &  37.35 &  28.02 &   31.57 &       35.6M \\
  PreLM        &  33.90 &  27.88 &  31.77 &  29.35 &  40.53 &  40.41 &  37.46 &  22.80 &  18.11 &  35.96 &  24.78 &  37.55 &  28.11 &   31.43 &       29.3M \\
  \hline
  PALM         &  34.73 &  28.59 &  32.22 &  29.83 &  40.31 &  40.91 &  37.61 &  23.27 &  18.74 &  37.00 &  25.29 &  37.59 &  28.75 &   31.91 &       35.6M \\
  PALM w/o SAE &  34.03 &  28.30 &  31.83 &  29.42 &  39.87 &  40.65 &  37.25 &  22.21 &  17.92 &  37.06 &  24.89 &  37.09 &  28.40 &   31.46 &       35.6M \\
      \bottomrule
      \end{tabular}
  }

  \caption{BLEU scores for models on IWSLT’14 dataset. Avg. is the average BLEU score for all language pairs while \#Paras is the count of model parameters.}
  \label{tab:iwslt}
  \vspace{-1em}
  \end{table*}

  \begin{figure*}[t]
    \centering
    \begin{minipage}[b]{0.35\textwidth}
    \centering
    \includegraphics[width=1\columnwidth]{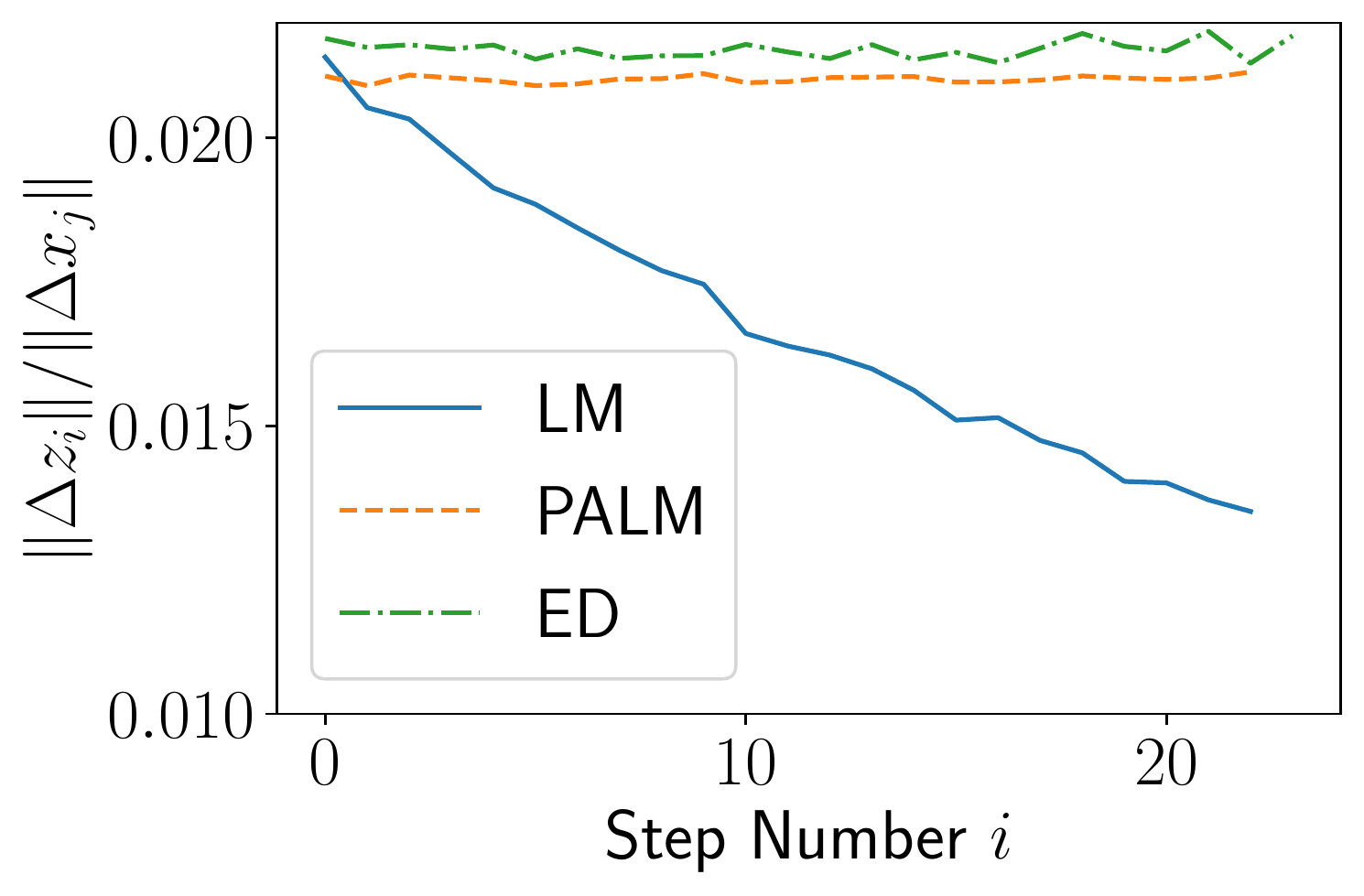}
    \caption{Numerical sensitivity analysis.\\\quad}
    \label{fig:jacobs}
    \end{minipage}
    \quad \quad \quad \quad \quad \quad 
    \begin{minipage}[b]{0.35\textwidth}
    \centering
    \includegraphics[width=1\columnwidth]{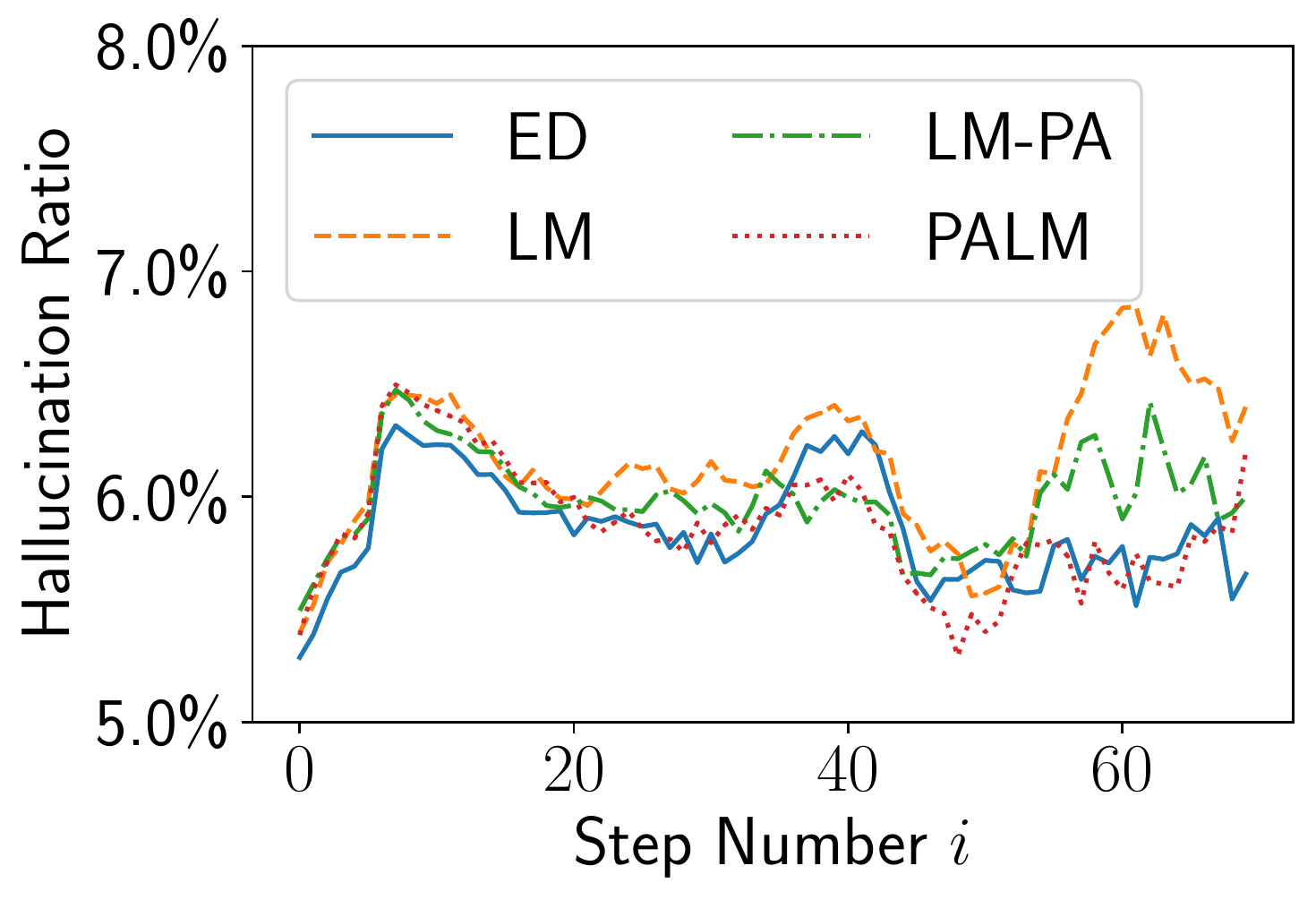}
    \caption{Stepwise hallucination analysis. Lower is better.}
    \label{fig:hallucination}
    \end{minipage}
  \end{figure*}

  \begin{table*}[t]
          \centering
          \small
          \begin{tabular}{@{~}l@{~}|@{~}c@{~}@{~}c@{~}@{~}c@{~}@{~}c@{~}|@{~}c@{~}}
            \toprule
            {} & Cs-En & En-Cs & Ro-En & En-Ro &  Avg. \\
            \hline
            ED           &  21.9 &  15.7 &  29.3 &  20.4 &  21.8 \\
            LM           &  16.2 &  11.0 &  25.1 &  15.9 &  17.1 \\
            PALM         &  21.1 &  15.1 &  28.3 &  19.3 &  21.0 \\
            PALM w/o SAE &  18.6 &  12.9 &  26.5 &  18.0 &  19.0 \\
            \bottomrule
            \end{tabular}
          
          \caption{BLEU score for WMT'16 dataset.}
          \label{tab:wmt}
  \end{table*}

\begin{table}[t]
    \centering

    \begin{minipage}[b]{0.45\textwidth}
      \small
      \centering
      \begin{tabular}{@{~}l@{~}|@{~}c@{~}@{~}c@{~}@{~}c@{~}@{~}c@{~}@{~}c@{~}@{~}c@{~}}
        \toprule
        {} &     LM &  LM-PA &   PALM &     ED &   Gold & $\Delta L$ \\
      \hline
      De-En &  21.00 &  21.37 &  21.30 &  21.08 &  22.25 &      0.30 \\
      En-De &  22.02 &  22.24 &  22.11 &  22.30 &  22.98 &      0.08 \\
      It-En &  22.35 &  22.40 &  22.37 &  22.84 &  23.88 &      0.02 \\
      En-It &  22.82 &  22.90 &  22.91 &  22.65 &  23.26 &      0.09 \\
      En-Fr &  23.70 &  24.29 &  24.23 &  24.00 &  24.01 &      0.53 \\
      Es-En &  21.16 &  21.46 &  21.35 &  21.88 &  22.48 &      0.19 \\
      En-Es &  21.66 &  21.97 &  21.96 &  22.06 &  22.03 &      0.30 \\
      Ru-En &  20.69 &  20.99 &  20.88 &  20.88 &  23.17 &      0.19 \\
      En-Ru &  22.21 &  22.71 &  22.72 &  22.87 &  22.96 &      0.50 \\
      He-En &  19.71 &  20.08 &  20.13 &  20.33 &  21.30 &      0.41 \\
      En-He &  19.61 &  20.16 &  20.03 &  20.31 &  20.66 &      0.42 \\
      Ro-En &  19.54 &  20.13 &  19.88 &  20.01 &  21.34 &      0.33 \\
      En-Ro &  22.38 &  22.72 &  22.61 &  22.82 &  23.03 &      0.24 \\
      \hline
      Avg.  &  21.45 &  21.80 &  21.73 &  21.85 &  22.57 &      0.28 \\
          \bottomrule
        \end{tabular}

      \caption{Early stop effect. $\Delta L$ is calculated by length of PALM minus length of LM. }
      \label{tab:sentlen}
    
    \end{minipage}
    \vspace{-0em}
  \end{table}

  \begin{table}[t]
    \centering
    \small
    \begin{tabular}{l|ccccc}
    \toprule
    {} & METEOR & ROUGE$_L$ & CIDEr &  NIST &   BLEU \\
    \hline
    ED         &  0.441 &   0.725 &  4.08 &  11.0 &  0.616 \\
    LM         &  0.452 &   0.732 &  4.24 &  11.2 &  0.629 \\
    LM w/o SAE &  0.438 &   0.716 &  4.10 &  10.9 &  0.607 \\
    PALM       &  0.457 &   0.736 &  4.27 &  11.3 &  0.632 \\
    \bottomrule
    \end{tabular}  
  \caption{Results on WebNLG dataset.}
  \label{tab:webnlg}
  \end{table}

  \begin{table}[t]
    \centering
    \small    
      \begin{tabular}{l|ccccc}
      \toprule
      {} & METEOR & ROUGE$_L$ & CIDEr &  NIST &   BLEU \\
      \hline
      ED         &  0.396 &   0.630 &  1.62 &  8.07 &  0.597 \\
      LM         &  0.423 &   0.661 &  2.00 &  7.93 &  0.604 \\
      LM w/o SAE &  0.405 &   0.646 &  1.92 &  7.93 &  0.602 \\
      PALM       &  0.449 &   0.688 &  2.25 &  8.46 &  0.657 \\
      \bottomrule
      \end{tabular} 
  \caption{Results on E2E dataset.}
  \label{tab:e2e}
  \end{table}
  
\textbf{Main Results.} The experimental results for IWSLT'14 are shown in \Cref{tab:iwslt}. We can draw the following conclusions. (1) Compared with the ED framework, the LM has worse performance due to the defects we discussed in Section \ref{sec:attention-degeneration-problem}. (2) Compared with the LM model, our proposed PALM structure achieved better performance which indicates that the attention degeneration is alleviated. (3) The PreLM improves the performance since the source text hidden representation depends on the whole source sequence instead of the previous context. (4) The LE component helps improve the performance as it provides extra language information for the model. (5) the LM-SPE model performs better than the traditional LM model using CPE. This observation verifies our analysis in Section \ref{sec:red}. (6) From the ablation study, the performance decreases without the SAE component. It shows that the SAE component can alleviate the overfitting problem and thus improves the performance. (7) Compared with the ED model, LM, and our proposed PALM reduce the parameter number significantly.

\begin{table}[t]
    \centering
    \begin{minipage}[b]{0.45\textwidth}
        \small
            \begin{tabular}{l|ccccc}
            \toprule
            {} & METEOR & ROUGE$_L$ & CIDEr &  NIST &   BLEU \\
            \hline
            ED  &  0.347 &   0.636 &  3.64 &  9.08 &  0.427 \\
            LM   &  0.352 &   0.648 &  3.77 &  7.83 &  0.382 \\
            PALM &  0.364 &   0.662 &  4.03 &   8.9 &  0.437 \\
            \bottomrule
            \end{tabular}
      \caption{Results on WITA dataset.}
      \label{tab:wita}
    \end{minipage}
    \quad \quad \quad
    \begin{minipage}[b]{0.35\textwidth}
        \small
  \centering
    \begin{tabular}{l@{~}|@{~}c@{~}@{~}c@{~}@{~}c}
      \toprule
      {} & ROUGE$_1$ & ROUGE$_2$ & ROUGE$_L$ \\
      \hline
      ED  &   0.284 &  0.0877 &    0.245 \\
      LM   &   0.279 &  0.0711 &   0.251 \\
      PALM &   0.317 &  0.0953 &   0.282 \\
      \bottomrule
      \end{tabular}
      \captionof{table}{Results on XSUM dataset.}
  \label{tab:xsum}
    \end{minipage}
  \end{table}

\textbf{Numerical Sensitivity Analysis.} In Theorem \ref{thm:jacobs}, we prove that the upper bound of the sensitivity in the unidirectional cross attention decreases as the current step number increases. To experimentally show that the output vector becomes less and less sensitive to the source sequence as the step number grows, we conduct a numerical analysis by directly imposing a random small perturbation vector $\Delta x_j$ on the source vector $x_j$ of the attention layer. We take the norm of the output vector deviation $\|\Delta z_i\|$ at the $i$th step divided by the input vector change norm $\|\Delta x_j\|$. This quantity is upper bounded by the sensitivity as shown in Proposition \ref{prop:pertube}. The results are shown in \Cref{fig:jacobs}. It can be concluded from the results that the sensitivity of the LM decreases as the step increases which leads to the model ignoring the input sequence gradually. It can also justify the conclusion drawn from Theorem \ref{thm:jacobs} that the ED model does not have this problem while our proposed PALM can also alleviate this problem.

\textbf{Stepwise Hallucination Analysis.}
As implied by Theorem \ref{thm:jacobs} and the above experiments, the sensitivity decreases as the step number increases in LM. But how does the sensitivity affect the generating performance? We conduct the stepwise hallucination analysis to answer this question. Generation models prone to generate unrelated content beyond the input data and this phenomenon has been recognized as the hallucination problem \cite{tian2019sticking,nie2019simple,liu2021token,rebuffel2021controlling} in many recent research works. Intuitively, the attention degeneration problem is likely to exacerbate the hallucination problem when the generation step grows. This is because less and less attention is focused on the input data and the models are more likely to generate text groundlessly. As shown in \Cref{fig:hallucination}, we conduct an experiment to analyze the hallucination problem at different generation steps. Similar to \cite{rebuffel2021controlling}, we use the Hallucination Ratio (HR) to measure the hallucination problem. 
HR measures the proportion of the hallucination content. The HR for the $i$th position is defined as $H_i=1-\sum_j^T \max(\mathbbm{1}(t_{ji} \in r_j), C[s_{j1},t_{ji}], \cdots, C[s_{j|s_j|},t_{ji}])/\sum_j^{T} \mathbbm{1}(|t_j|\ge i),$ where $T$ is the target sentence count. $s_{ji}$ and $t_{ji}$ are the $i$th word in the $j$th source and generated sentence respectively. $|t_j|$ is the length of the $j$th generated sentence. $r_j$ is the $j$th gold standard sentence. $C[p,q]=\mathtt{Sigmoid}((\sum_j^T \mathbbm{1}(p\in s_j)\cdot \mathbbm{1}(q\in t_j)-\alpha) / \beta)$ is the normalized alignment score of the $p$th and $q$th words in the dictionary, where $\alpha$ and $\beta$ are sigmoid shape parameters. High HR indicates the hallucination problem is prone to occur.
It can be observed that (1) HR changes moderately between 5\% and 7\% as $i$ grows in all models. (2) LM has the highest (worst) HR score especially when $i$ grows larger. When $i>50$, LM has a significant high HR score than other models which implies that the hallucination problem gets worse. (3) Simply adding the PA component on LM can alleviate the hallucination problem and lower the HR score. It indicates that the attention degeneration problem exacerbates the hallucination problem and it can be alleviated with our proposed PA component. (4) Our proposed PALM achieves a low HR score close to the ED model. It performs well even $i$ is large.

\textbf{Early Stop Effect.} To further study how the attention degeneration problem affects the performance, we conduct a statistical analysis on the output length of the generated sequence. The results are shown in \Cref{tab:sentlen}. It can be concluded from the results that (1) The LM usually generates shorter sentences compared with the ED model and PALM structure while simply adding the PA component in LM alleviates the early stop phenomenon. It shows that the attention degeneration problem causes the early stop phenomenon and it can be alleviated by our proposed PA component. Due to the attention degeneration problem, the sensitivity on the input source vectors decreases as the step number grows. The model becomes less and less sensitive to the source sequence and it tends to stop generating because it gets less and less information about what to say. 
(2) On the other hand, our proposed PALM structure alleviates this problem. The average length of the PALM structure output is longer than the LM model and close to the ED model showing that it alleviates the attention degeneration problem and thus overcomes the early stop effect.

\textbf{Data-to-Text Generation.} To show that our analysis and the proposed method are also applicable in other seq2seq tasks, we conduct experiments on the data-to-text generation tasks. The results on WebNLG, E2E, and WITA datasets are shown in \Cref{tab:webnlg}, \ref{tab:e2e}, and \ref{tab:wita} respectively. It can be observed that (1) our proposed PALM outperforms the LM model which shows the effectiveness of the partial attention component. (2) The LM model outperforms the ED model with the help of SAE component. Different from the translation task, the input sequence of this task shares the same language as the target sequence. Therefore, LM's SAE component trains the model to recover source sequences and can thus alleviate the overfitting problem. This can be directly verfied with the results of the LM w/o SAE model. When the SAE component is removed, the performance is worse than the LM model in all three datasets.

\textbf{WMT'16 Results.} We conduct our experiments on WMT'16 \cite{bojar2016findings} dataset which is larger than IWSLT'14. The performance is shown in \Cref{tab:wmt}. It can be concluded from the results that (1) our proposed PALM also outperforms LM's results and is close the ED's performance which shows PALM's effectiveness in alleviating the attention degeneration problem; (2) interestingly, removing the SAE component improves the performance and closes the gap to the performance of ED. This may be because the dataset is larger and the overfitting problem is not serious.

\textbf{Summarization.} We conduct experiments on the summarization dataset XSUM to further verify our analysis and the results are shown in \Cref{tab:xsum}. It can be concluded from the results that (1) the ED model outperforms the LM model in ROUGE$_1$ and ROUGE$_2$. This is because the input sequence is much longer than the target sequence, the attention degeneration problem becomes the dominant problem. This result is consistent with the analysis we discussed in Section \ref{sec:attention-degeneration-problem}; (2) our PALM structure still outperforms the LM model which is also consistent with the previous analysis; 
%(2) the LM also outperforms the ED model which has been observed in many previous summarization research \cite{pilault2020extractive}. It may be caused by the fact that in the summarization data, the source sequences are always much longer than the target sequences. In the LM, the SAE component and the parameter sharing mechanism additionally incorporates the source sequences as the target to train to generate text while the ED structure is only trained with the much shorter target sequence to generate. Besides, other properties of LM including layer-wise coordination and etc. may also help improve the results.

\textbf{Stepwise Precision Analysis.} 
  As implied by Theorem \ref{thm:jacobs} and the above experiments, the sensitivity decreases as the step number increases in LM. We further conduct the stepwise precision analysis to show that the sensitivity directly affect the generating performance. We calculate the average word precision for the $i$th position as $A_i=\sum_j^{S} \mathbbm{1}(s_{ji} \in r_j)/\sum_j^{S} \mathbbm{1}(|s_j|\ge i),$ where $s_j$ in the $j$th generated sentence and $s_{ji}$ is the $i$th word in sentence $s_j$. $r_j$ is the $j$th reference sentence and there are $S$ sentences in total. The results are shown in \Cref{fig:stepprecision}. We have three observations. (1) $A_i$ changes moderately between 0.64 and 0.67 as $i$ grows which is determined by the word distribution. (2) $A_i$ for the LM model is slightly better than the ED framework when the step number is small. This is because when the step number is small, the attention degeneration problem is not the dominant factor for the performance and the LE, SAE, and other components play as the key role to make the performance better. (3) Contrastingly, $A_i$ for the LM model becomes worse as the step size gets larger when compared with the ED model. This observation shows that as the step number increases, the sensitivity decreases, and thus the model is likely to generate wrong words. (4) Our proposed PALM structure alleviates the attention degeneration problem and the precision distribution is similar to the ED model.
  
  \begin{figure}[t]
    \centering
    \includegraphics[width=0.4\columnwidth]{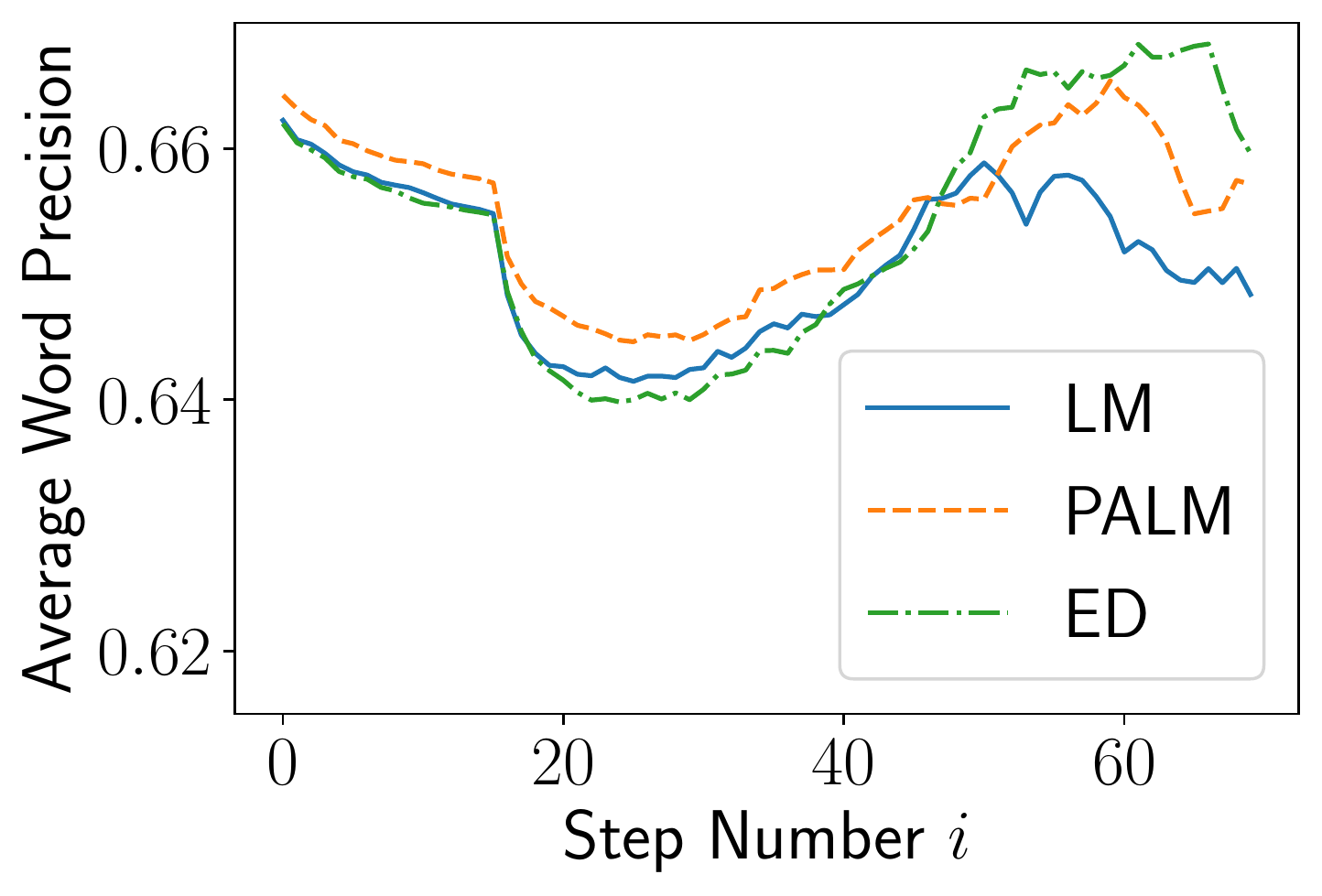}
    \caption{Stepwise precision analysis. $x$-axis is the generation step number and $y$-axis is the average word precision.}
    \label{fig:stepprecision}
  \end{figure}

\section{Conclusions}
We propose a detailed comparison between the Encoder-Decoder (ED) framework and the Language Model (LM) in the sequence-to-sequence (seq2seq) generation task. We propose to analyze a Regularized Encoder-Decoder (RED) framework that is equivalent to an LM but has an encoder and a decoder. We analyze the attention degeneration problem and conduct a theoretical analysis. Based on the analysis, we propose a novel partial attention language model. The experimental results validate our analysis and demonstrate the effectiveness of our proposed model.

\section{Limitations}
It should be noted that in Theorem \ref{thm:jacobs}, we utilize the PAC theory framework which holds with high probability. Unfortunately, it cannot guarantee all the cases satisfy the bound which means in some extreme cases, the bound will not be valid. This is a common short back for the PAC analysis and we use extensive experimental results to empirically show the effectiveness of the analysis. 

\section{Broader Impact Statement}
This paper focuses on comparing different architectures of the existing language model which does not explicitly involve any enthetic concerns. We do not make any new dataset and all used datasets are properly cited. This work does not cause any kind of safety or security concerns while it also does not raise any human rights concerns or environmental concerns. Therefore, there will be no negative societal impact on our work.

% Entries for the entire Anthology, followed by custom entries
\bibliography{reference}

\begin{thebibliography}{72}
\expandafter\ifx\csname natexlab\endcsname\relax\def\natexlab#1{#1}\fi

\bibitem[{Bahdanau et~al.(2015)Bahdanau, Cho, and Bengio}]{bahdanau2014neural}
Dzmitry Bahdanau, Kyunghyun Cho, and Yoshua Bengio. 2015.
\newblock Neural machine translation by jointly learning to align and
  translate.
\newblock In \emph{3rd International Conference on Learning Representations,
  {ICLR} 2015, San Diego, CA, USA, May 7-9, 2015, Conference Track
  Proceedings}.

\bibitem[{Banerjee and Lavie(2005)}]{banerjee2005meteor}
Satanjeev Banerjee and Alon Lavie. 2005.
\newblock {METEOR}: An automatic metric for {MT} evaluation with improved
  correlation with human judgments.
\newblock In \emph{Proceedings of the {ACL} Workshop on Intrinsic and Extrinsic
  Evaluation Measures for Machine Translation and/or Summarization}, pages
  65--72, Ann Arbor, Michigan. Association for Computational Linguistics.

\bibitem[{Belinkov et~al.(2017)Belinkov, Durrani, Dalvi, Sajjad, and
  Glass}]{belinkov2017neural}
Yonatan Belinkov, Nadir Durrani, Fahim Dalvi, Hassan Sajjad, and James Glass.
  2017.
\newblock What do neural machine translation models learn about morphology?
\newblock In \emph{Proceedings of the 55th Annual Meeting of the Association
  for Computational Linguistics (Volume 1: Long Papers)}, pages 861--872,
  Vancouver, Canada. Association for Computational Linguistics.

\bibitem[{Bengio et~al.(2000)Bengio, Ducharme, and Vincent}]{bengio2003neural}
Yoshua Bengio, R{\'{e}}jean Ducharme, and Pascal Vincent. 2000.
\newblock A neural probabilistic language model.
\newblock In \emph{Advances in Neural Information Processing Systems 13, Papers
  from Neural Information Processing Systems {(NIPS)} 2000, Denver, CO, {USA}},
  pages 932--938. {MIT} Press.

\bibitem[{Bojar et~al.(2016)Bojar, Chatterjee, Federmann, Graham, Haddow, Huck,
  Jimeno~Yepes, Koehn, Logacheva, Monz, Negri, N{\'e}v{\'e}ol, Neves, Popel,
  Post, Rubino, Scarton, Specia, Turchi, Verspoor, and
  Zampieri}]{bojar2016findings}
Ond{\v{r}}ej Bojar, Rajen Chatterjee, Christian Federmann, Yvette Graham, Barry
  Haddow, Matthias Huck, Antonio Jimeno~Yepes, Philipp Koehn, Varvara
  Logacheva, Christof Monz, Matteo Negri, Aur{\'e}lie N{\'e}v{\'e}ol, Mariana
  Neves, Martin Popel, Matt Post, Raphael Rubino, Carolina Scarton, Lucia
  Specia, Marco Turchi, Karin Verspoor, and Marcos Zampieri. 2016.
\newblock Findings of the 2016 conference on machine translation.
\newblock In \emph{Proceedings of the First Conference on Machine Translation:
  Volume 2, Shared Task Papers}, pages 131--198, Berlin, Germany. Association
  for Computational Linguistics.

\bibitem[{Brown et~al.(2020)Brown, Mann, Ryder, Subbiah, Kaplan, Dhariwal,
  Neelakantan, Shyam, Sastry, Askell, Agarwal, Herbert{-}Voss, Krueger,
  Henighan, Child, Ramesh, Ziegler, Wu, Winter, Hesse, Chen, Sigler, Litwin,
  Gray, Chess, Clark, Berner, McCandlish, Radford, Sutskever, and
  Amodei}]{brown2020language}
Tom~B. Brown, Benjamin Mann, Nick Ryder, Melanie Subbiah, Jared Kaplan,
  Prafulla Dhariwal, Arvind Neelakantan, Pranav Shyam, Girish Sastry, Amanda
  Askell, Sandhini Agarwal, Ariel Herbert{-}Voss, Gretchen Krueger, Tom
  Henighan, Rewon Child, Aditya Ramesh, Daniel~M. Ziegler, Jeffrey Wu, Clemens
  Winter, Christopher Hesse, Mark Chen, Eric Sigler, Mateusz Litwin, Scott
  Gray, Benjamin Chess, Jack Clark, Christopher Berner, Sam McCandlish, Alec
  Radford, Ilya Sutskever, and Dario Amodei. 2020.
\newblock Language models are few-shot learners.
\newblock In \emph{Advances in Neural Information Processing Systems 33: Annual
  Conference on Neural Information Processing Systems 2020, NeurIPS 2020,
  December 6-12, 2020, virtual}.

\bibitem[{Castro~Ferreira et~al.(2018)Castro~Ferreira, Moussallem, Krahmer, and
  Wubben}]{ferreira2018enriching}
Thiago Castro~Ferreira, Diego Moussallem, Emiel Krahmer, and Sander Wubben.
  2018.
\newblock Enriching the {W}eb{NLG} corpus.
\newblock In \emph{Proceedings of the 11th International Conference on Natural
  Language Generation}, pages 171--176, Tilburg University, The Netherlands.
  Association for Computational Linguistics.

\bibitem[{Cettolo et~al.(2014)Cettolo, Niehues, St{\"u}ker, Bentivogli, and
  Federico}]{cettolo2014report}
Mauro Cettolo, J~Niehues, S~St{\"u}ker, Luisa Bentivogli, and Marcello
  Federico. 2014.
\newblock Report on the 11th iwslt evaluation campaign, iwslt 2014.
\newblock In \emph{IWSLT-International Workshop on Spoken Language Processing},
  pages 2--17. Marcello Federico, Sebastian St{\"u}ker, Fran{\c{c}}ois Yvon.

\bibitem[{Cheng and Lapata(2016)}]{cheng2016neural}
Jianpeng Cheng and Mirella Lapata. 2016.
\newblock Neural summarization by extracting sentences and words.
\newblock In \emph{Proceedings of the 54th Annual Meeting of the Association
  for Computational Linguistics (Volume 1: Long Papers)}, pages 484--494,
  Berlin, Germany. Association for Computational Linguistics.

\bibitem[{Cho et~al.(2014)Cho, van Merri{\"e}nboer, Gulcehre, Bahdanau,
  Bougares, Schwenk, and Bengio}]{cho2014learning}
Kyunghyun Cho, Bart van Merri{\"e}nboer, Caglar Gulcehre, Dzmitry Bahdanau,
  Fethi Bougares, Holger Schwenk, and Yoshua Bengio. 2014.
\newblock Learning phrase representations using {RNN} encoder{--}decoder for
  statistical machine translation.
\newblock In \emph{Proceedings of the 2014 Conference on Empirical Methods in
  Natural Language Processing ({EMNLP})}, pages 1724--1734, Doha, Qatar.
  Association for Computational Linguistics.

\bibitem[{Chopra et~al.(2016)Chopra, Auli, and Rush}]{chopra2016abstractive}
Sumit Chopra, Michael Auli, and Alexander~M. Rush. 2016.
\newblock Abstractive sentence summarization with attentive recurrent neural
  networks.
\newblock In \emph{Proceedings of the 2016 Conference of the North {A}merican
  Chapter of the Association for Computational Linguistics: Human Language
  Technologies}, pages 93--98, San Diego, California. Association for
  Computational Linguistics.

\bibitem[{Chowdhery et~al.(2022)Chowdhery, Narang, Devlin, Bosma, Mishra,
  Roberts, Barham, Chung, Sutton, Gehrmann et~al.}]{chowdhery2022palm}
Aakanksha Chowdhery, Sharan Narang, Jacob Devlin, Maarten Bosma, Gaurav Mishra,
  Adam Roberts, Paul Barham, Hyung~Won Chung, Charles Sutton, Sebastian
  Gehrmann, et~al. 2022.
\newblock Palm: Scaling language modeling with pathways.
\newblock \emph{arXiv preprint arXiv:2204.02311}.

\bibitem[{Conneau et~al.(2020)Conneau, Khandelwal, Goyal, Chaudhary, Wenzek,
  Guzm{\'a}n, Grave, Ott, Zettlemoyer, and Stoyanov}]{conneau2020unsupervised}
Alexis Conneau, Kartikay Khandelwal, Naman Goyal, Vishrav Chaudhary, Guillaume
  Wenzek, Francisco Guzm{\'a}n, Edouard Grave, Myle Ott, Luke Zettlemoyer, and
  Veselin Stoyanov. 2020.
\newblock Unsupervised cross-lingual representation learning at scale.
\newblock In \emph{Proceedings of the 58th Annual Meeting of the Association
  for Computational Linguistics}, pages 8440--8451, Online. Association for
  Computational Linguistics.

\bibitem[{Conneau and Lample(2019)}]{conneau2019cross}
Alexis Conneau and Guillaume Lample. 2019.
\newblock Cross-lingual language model pretraining.
\newblock In \emph{Advances in Neural Information Processing Systems 32: Annual
  Conference on Neural Information Processing Systems 2019, NeurIPS 2019,
  December 8-14, 2019, Vancouver, BC, Canada}, pages 7057--7067.

\bibitem[{Dehghani et~al.(2019)Dehghani, Gouws, Vinyals, Uszkoreit, and
  Kaiser}]{dehghani2018universal}
Mostafa Dehghani, Stephan Gouws, Oriol Vinyals, Jakob Uszkoreit, and Lukasz
  Kaiser. 2019.
\newblock Universal transformers.
\newblock In \emph{7th International Conference on Learning Representations,
  {ICLR} 2019, New Orleans, LA, USA, May 6-9, 2019}. OpenReview.net.

\bibitem[{Deng et~al.(2023)Deng, Sun, Zhang, Cheng, and Huang}]{deng2023recent}
Jiawen Deng, Hao Sun, Zhexin Zhang, Jiale Cheng, and Minlie Huang. 2023.
\newblock Recent advances towards safe, responsible, and moral dialogue
  systems: A survey.
\newblock \emph{arXiv preprint arXiv:2302.09270}.

\bibitem[{Doddington(2002)}]{doddington2002automatic}
George Doddington. 2002.
\newblock Automatic evaluation of machine translation quality using n-gram
  co-occurrence statistics.
\newblock In \emph{Proceedings of the Second International Conference on Human
  Language Technology Research}, pages 138--145.

\bibitem[{Dong et~al.(2019)Dong, Yang, Wang, Wei, Liu, Wang, Gao, Zhou, and
  Hon}]{dong2019unified}
Li~Dong, Nan Yang, Wenhui Wang, Furu Wei, Xiaodong Liu, Yu~Wang, Jianfeng Gao,
  Ming Zhou, and Hsiao{-}Wuen Hon. 2019.
\newblock Unified language model pre-training for natural language
  understanding and generation.
\newblock In \emph{Advances in Neural Information Processing Systems 32: Annual
  Conference on Neural Information Processing Systems 2019, December 8-14,
  2019, Vancouver, BC, Canada}, pages 13042--13054.

\bibitem[{Dong et~al.(2021)Dong, Cordonnier, and Loukas}]{dong2021attention}
Yihe Dong, Jean-Baptiste Cordonnier, and Andreas Loukas. 2021.
\newblock Attention is not all you need: Pure attention loses rank doubly
  exponentially with depth.
\newblock In \emph{International Conference on Machine Learning}, pages
  2793--2803. PMLR.

\bibitem[{Du{\v{s}}ek et~al.(2019)Du{\v{s}}ek, Howcroft, and
  Rieser}]{duvsek2019semantic}
Ond{\v{r}}ej Du{\v{s}}ek, David~M. Howcroft, and Verena Rieser. 2019.
\newblock Semantic noise matters for neural natural language generation.
\newblock In \emph{Proceedings of the 12th International Conference on Natural
  Language Generation}, pages 421--426, Tokyo, Japan. Association for
  Computational Linguistics.

\bibitem[{Fu et~al.(2021)Fu, Lam, So, and Shi}]{fu2021theoretical}
Zihao Fu, Wai Lam, Anthony Man-Cho So, and Bei Shi. 2021.
\newblock A theoretical analysis of the repetition problem in text generation.
\newblock In \emph{Proceedings of the AAAI Conference on Artificial
  Intelligence}, volume~35, pages 12848--12856.

\bibitem[{Fu et~al.(2020)Fu, Shi, Lam, Bing, and Liu}]{fu2020partially}
Zihao Fu, Bei Shi, Wai Lam, Lidong Bing, and Zhiyuan Liu. 2020.
\newblock Partially-aligned data-to-text generation with distant supervision.
\newblock In \emph{Proceedings of the 2020 Conference on Empirical Methods in
  Natural Language Processing (EMNLP)}, pages 9183--9193, Online. Association
  for Computational Linguistics.

\bibitem[{Gardent et~al.(2017{\natexlab{a}})Gardent, Shimorina, Narayan, and
  Perez-Beltrachini}]{gardent2017creating}
Claire Gardent, Anastasia Shimorina, Shashi Narayan, and Laura
  Perez-Beltrachini. 2017{\natexlab{a}}.
\newblock Creating training corpora for {NLG} micro-planners.
\newblock In \emph{Proceedings of the 55th Annual Meeting of the Association
  for Computational Linguistics (Volume 1: Long Papers)}, pages 179--188,
  Vancouver, Canada. Association for Computational Linguistics.

\bibitem[{Gardent et~al.(2017{\natexlab{b}})Gardent, Shimorina, Narayan, and
  Perez-Beltrachini}]{gardent2017webnlg}
Claire Gardent, Anastasia Shimorina, Shashi Narayan, and Laura
  Perez-Beltrachini. 2017{\natexlab{b}}.
\newblock The {W}eb{NLG} challenge: Generating text from {RDF} data.
\newblock In \emph{Proceedings of the 10th International Conference on Natural
  Language Generation}, pages 124--133, Santiago de Compostela, Spain.
  Association for Computational Linguistics.

\bibitem[{Gehring et~al.(2017)Gehring, Auli, Grangier, Yarats, and
  Dauphin}]{gehring2017convolutional}
Jonas Gehring, Michael Auli, David Grangier, Denis Yarats, and Yann~N. Dauphin.
  2017.
\newblock Convolutional sequence to sequence learning.
\newblock In \emph{Proceedings of the 34th International Conference on Machine
  Learning, {ICML} 2017, Sydney, NSW, Australia, 6-11 August 2017}, volume~70
  of \emph{Proceedings of Machine Learning Research}, pages 1243--1252. {PMLR}.

\bibitem[{Ghosh et~al.(2017)Ghosh, Chollet, Laksana, Morency, and
  Scherer}]{ghosh2017affect}
Sayan Ghosh, Mathieu Chollet, Eugene Laksana, Louis-Philippe Morency, and
  Stefan Scherer. 2017.
\newblock Affect-{LM}: A neural language model for customizable affective text
  generation.
\newblock In \emph{Proceedings of the 55th Annual Meeting of the Association
  for Computational Linguistics (Volume 1: Long Papers)}, pages 634--642,
  Vancouver, Canada. Association for Computational Linguistics.

\bibitem[{He et~al.(2018)He, Tan, Xia, He, Qin, Chen, and Liu}]{he2018layer}
Tianyu He, Xu~Tan, Yingce Xia, Di~He, Tao Qin, Zhibo Chen, and Tie{-}Yan Liu.
  2018.
\newblock Layer-wise coordination between encoder and decoder for neural
  machine translation.
\newblock In \emph{Advances in Neural Information Processing Systems 31: Annual
  Conference on Neural Information Processing Systems 2018, NeurIPS 2018,
  December 3-8, 2018, Montr{\'{e}}al, Canada}, pages 7955--7965.

\bibitem[{Johnson et~al.(2017)Johnson, Schuster, Le, Krikun, Wu, Chen, Thorat,
  Vi{\'e}gas, Wattenberg, Corrado, Hughes, and Dean}]{johnson2017google}
Melvin Johnson, Mike Schuster, Quoc~V. Le, Maxim Krikun, Yonghui Wu, Zhifeng
  Chen, Nikhil Thorat, Fernanda Vi{\'e}gas, Martin Wattenberg, Greg Corrado,
  Macduff Hughes, and Jeffrey Dean. 2017.
\newblock {G}oogle{'}s multilingual neural machine translation system: Enabling
  zero-shot translation.
\newblock \emph{Transactions of the Association for Computational Linguistics},
  5:339--351.

\bibitem[{Kim et~al.(2020)Kim, Papamakarios, and Mnih}]{kim2020lipschitz}
Hyunjik Kim, George Papamakarios, and Andriy Mnih. 2020.
\newblock The lipschitz constant of self-attention.
\newblock \emph{arXiv preprint arXiv:2006.04710}.

\bibitem[{Klein et~al.(2017)Klein, Kim, Deng, Senellart, and Rush}]{opennmt}
Guillaume Klein, Yoon Kim, Yuntian Deng, Jean Senellart, and Alexander Rush.
  2017.
\newblock {O}pen{NMT}: Open-source toolkit for neural machine translation.
\newblock In \emph{Proceedings of {ACL} 2017, System Demonstrations}, pages
  67--72, Vancouver, Canada. Association for Computational Linguistics.

\bibitem[{Koehn and Knowles(2017)}]{koehn2017six}
Philipp Koehn and Rebecca Knowles. 2017.
\newblock Six challenges for neural machine translation.
\newblock In \emph{Proceedings of the First Workshop on Neural Machine
  Translation}, pages 28--39, Vancouver. Association for Computational
  Linguistics.

\bibitem[{Lan et~al.(2020)Lan, Chen, Goodman, Gimpel, Sharma, and
  Soricut}]{lan2019albert}
Zhenzhong Lan, Mingda Chen, Sebastian Goodman, Kevin Gimpel, Piyush Sharma, and
  Radu Soricut. 2020.
\newblock {ALBERT:} {A} lite {BERT} for self-supervised learning of language
  representations.
\newblock In \emph{8th International Conference on Learning Representations,
  {ICLR} 2020, Addis Ababa, Ethiopia, April 26-30, 2020}. OpenReview.net.

\bibitem[{Lin(2004)}]{lin2004rouge}
Chin-Yew Lin. 2004.
\newblock {ROUGE}: A package for automatic evaluation of summaries.
\newblock In \emph{Text Summarization Branches Out}, pages 74--81, Barcelona,
  Spain. Association for Computational Linguistics.

\bibitem[{Liu et~al.(2018)Liu, Saleh, Pot, Goodrich, Sepassi, Kaiser, and
  Shazeer}]{liu2018generating}
Peter~J. Liu, Mohammad Saleh, Etienne Pot, Ben Goodrich, Ryan Sepassi, Lukasz
  Kaiser, and Noam Shazeer. 2018.
\newblock Generating wikipedia by summarizing long sequences.
\newblock In \emph{6th International Conference on Learning Representations,
  {ICLR} 2018, Vancouver, BC, Canada, April 30 - May 3, 2018, Conference Track
  Proceedings}. OpenReview.net.

\bibitem[{Liu et~al.(2021)Liu, Zhang, Brockett, Mao, Sui, Chen, and
  Dolan}]{liu2021token}
Tianyu Liu, Yizhe Zhang, Chris Brockett, Yi~Mao, Zhifang Sui, Weizhu Chen, and
  Bill Dolan. 2021.
\newblock A token-level reference-free hallucination detection benchmark for
  free-form text generation.
\newblock \emph{arXiv preprint arXiv:2104.08704}.

\bibitem[{Luong et~al.(2015)Luong, Pham, and Manning}]{luong2015effective}
Thang Luong, Hieu Pham, and Christopher~D. Manning. 2015.
\newblock Effective approaches to attention-based neural machine translation.
\newblock In \emph{Proceedings of the 2015 Conference on Empirical Methods in
  Natural Language Processing}, pages 1412--1421, Lisbon, Portugal. Association
  for Computational Linguistics.

\bibitem[{Mikolov et~al.(2010)Mikolov, Karafi{\'a}t, Burget,
  {\v{C}}ernock{\`y}, and Khudanpur}]{mikolov2010recurrent}
Tom{\'a}{\v{s}} Mikolov, Martin Karafi{\'a}t, Luk{\'a}{\v{s}} Burget, Jan
  {\v{C}}ernock{\`y}, and Sanjeev Khudanpur. 2010.
\newblock Recurrent neural network based language model.
\newblock In \emph{Eleventh annual conference of the international speech
  communication association}.

\bibitem[{Mikolov et~al.(2011)Mikolov, Kombrink, Burget, {\v{C}}ernock{\`y},
  and Khudanpur}]{mikolov2011extensions}
Tom{\'a}{\v{s}} Mikolov, Stefan Kombrink, Luk{\'a}{\v{s}} Burget, Jan
  {\v{C}}ernock{\`y}, and Sanjeev Khudanpur. 2011.
\newblock Extensions of recurrent neural network language model.
\newblock In \emph{2011 IEEE international conference on acoustics, speech and
  signal processing (ICASSP)}, pages 5528--5531. IEEE.

\bibitem[{Mikolov and Zweig(2012)}]{mikolov2012context}
Tomas Mikolov and Geoffrey Zweig. 2012.
\newblock Context dependent recurrent neural network language model.
\newblock In \emph{2012 IEEE Spoken Language Technology Workshop (SLT)}, pages
  234--239.

\bibitem[{Nallapati et~al.(2016)Nallapati, Zhou, dos Santos, Gul{\c{c}}ehre,
  and Xiang}]{nallapati2016abstractive}
Ramesh Nallapati, Bowen Zhou, Cicero dos Santos, {\c{C}}a{\u{g}}lar
  Gul{\c{c}}ehre, and Bing Xiang. 2016.
\newblock Abstractive text summarization using sequence-to-sequence {RNN}s and
  beyond.
\newblock In \emph{Proceedings of The 20th {SIGNLL} Conference on Computational
  Natural Language Learning}, pages 280--290, Berlin, Germany. Association for
  Computational Linguistics.

\bibitem[{Narayan et~al.(2018)Narayan, Cohen, and Lapata}]{narayan2018don}
Shashi Narayan, Shay~B. Cohen, and Mirella Lapata. 2018.
\newblock Don{'}t give me the details, just the summary! topic-aware
  convolutional neural networks for extreme summarization.
\newblock In \emph{Proceedings of the 2018 Conference on Empirical Methods in
  Natural Language Processing}, pages 1797--1807, Brussels, Belgium.
  Association for Computational Linguistics.

\bibitem[{Nie et~al.(2019)Nie, Yao, Wang, Pan, and Lin}]{nie2019simple}
Feng Nie, Jin-Ge Yao, Jinpeng Wang, Rong Pan, and Chin-Yew Lin. 2019.
\newblock A simple recipe towards reducing hallucination in neural surface
  realisation.
\newblock In \emph{Proceedings of the 57th Annual Meeting of the Association
  for Computational Linguistics}, pages 2673--2679, Florence, Italy.
  Association for Computational Linguistics.

\bibitem[{Novikova et~al.(2017)Novikova, Du{\v{s}}ek, and
  Rieser}]{novikova2017e2e}
Jekaterina Novikova, Ond{\v{r}}ej Du{\v{s}}ek, and Verena Rieser. 2017.
\newblock The {E}2{E} dataset: New challenges for end-to-end generation.
\newblock In \emph{Proceedings of the 18th Annual {SIG}dial Meeting on
  Discourse and Dialogue}, pages 201--206, Saarbr{\"u}cken, Germany.
  Association for Computational Linguistics.

\bibitem[{Ott et~al.(2019)Ott, Edunov, Baevski, Fan, Gross, Ng, Grangier, and
  Auli}]{ott2019fairseq}
Myle Ott, Sergey Edunov, Alexei Baevski, Angela Fan, Sam Gross, Nathan Ng,
  David Grangier, and Michael Auli. 2019.
\newblock fairseq: A fast, extensible toolkit for sequence modeling.
\newblock In \emph{Proceedings of the 2019 Conference of the North {A}merican
  Chapter of the Association for Computational Linguistics (Demonstrations)},
  pages 48--53, Minneapolis, Minnesota. Association for Computational
  Linguistics.

\bibitem[{Ouyang et~al.(2022)Ouyang, Wu, Jiang, Almeida, Wainwright, Mishkin,
  Zhang, Agarwal, Slama, Ray et~al.}]{ouyang2022training}
Long Ouyang, Jeffrey Wu, Xu~Jiang, Diogo Almeida, Carroll Wainwright, Pamela
  Mishkin, Chong Zhang, Sandhini Agarwal, Katarina Slama, Alex Ray, et~al.
  2022.
\newblock Training language models to follow instructions with human feedback.
\newblock \emph{Advances in Neural Information Processing Systems},
  35:27730--27744.

\bibitem[{Papineni et~al.(2002)Papineni, Roukos, Ward, and
  Zhu}]{papineni2002bleu}
Kishore Papineni, Salim Roukos, Todd Ward, and Wei-Jing Zhu. 2002.
\newblock {B}leu: a method for automatic evaluation of machine translation.
\newblock In \emph{Proceedings of the 40th Annual Meeting of the Association
  for Computational Linguistics}, pages 311--318, Philadelphia, Pennsylvania,
  USA. Association for Computational Linguistics.

\bibitem[{Paulus et~al.(2018)Paulus, Xiong, and Socher}]{paulus2018deep}
Romain Paulus, Caiming Xiong, and Richard Socher. 2018.
\newblock A deep reinforced model for abstractive summarization.
\newblock In \emph{6th International Conference on Learning Representations,
  {ICLR} 2018, Vancouver, BC, Canada, April 30 - May 3, 2018, Conference Track
  Proceedings}. OpenReview.net.

\bibitem[{Perez-Beltrachini et~al.(2016)Perez-Beltrachini, Sayed, and
  Gardent}]{perez2016building}
Laura Perez-Beltrachini, Rania Sayed, and Claire Gardent. 2016.
\newblock Building {RDF} content for data-to-text generation.
\newblock In \emph{Proceedings of {COLING} 2016, the 26th International
  Conference on Computational Linguistics: Technical Papers}, pages 1493--1502,
  Osaka, Japan. The COLING 2016 Organizing Committee.

\bibitem[{Peters et~al.(2018)Peters, Neumann, Iyyer, Gardner, Clark, Lee, and
  Zettlemoyer}]{peters2018deep}
Matthew Peters, Mark Neumann, Mohit Iyyer, Matt Gardner, Christopher Clark,
  Kenton Lee, and Luke Zettlemoyer. 2018.
\newblock Deep contextualized word representations.
\newblock In \emph{Proceedings of the 2018 Conference of the North {A}merican
  Chapter of the Association for Computational Linguistics: Human Language
  Technologies, Volume 1 (Long Papers)}, pages 2227--2237, New Orleans,
  Louisiana. Association for Computational Linguistics.

\bibitem[{Pilault et~al.(2020)Pilault, Li, Subramanian, and
  Pal}]{pilault2020extractive}
Jonathan Pilault, Raymond Li, Sandeep Subramanian, and Christopher Pal. 2020.
\newblock On extractive and abstractive neural document summarization with
  transformer language models.
\newblock In \emph{Proceedings of the 2020 Conference on Empirical Methods in
  Natural Language Processing (EMNLP)}, pages 9308--9319.

\bibitem[{Radford et~al.(2018)Radford, Narasimhan, Salimans, Sutskever
  et~al.}]{radford2018improving}
Alec Radford, Karthik Narasimhan, Tim Salimans, Ilya Sutskever, et~al. 2018.
\newblock Improving language understanding by generative pre-training.

\bibitem[{Radford et~al.(2019)Radford, Wu, Child, Luan, Amodei, and
  Sutskever}]{radford2019language}
Alec Radford, Jeffrey Wu, Rewon Child, David Luan, Dario Amodei, and Ilya
  Sutskever. 2019.
\newblock Language models are unsupervised multitask learners.

\bibitem[{Raffel et~al.(2020)Raffel, Shazeer, Roberts, Lee, Narang, Matena,
  Zhou, Li, and Liu}]{raffel2020exploring}
Colin Raffel, Noam Shazeer, Adam Roberts, Katherine Lee, Sharan Narang, Michael
  Matena, Yanqi Zhou, Wei Li, and Peter~J Liu. 2020.
\newblock Exploring the limits of transfer learning with a unified text-to-text
  transformer.
\newblock \emph{Journal of Machine Learning Research}, 21(140):1--67.

\bibitem[{Rebuffel et~al.(2021)Rebuffel, Roberti, Soulier, Scoutheeten,
  Cancelliere, and Gallinari}]{rebuffel2021controlling}
Cl{\'e}ment Rebuffel, Marco Roberti, Laure Soulier, Geoffrey Scoutheeten,
  Rossella Cancelliere, and Patrick Gallinari. 2021.
\newblock Controlling hallucinations at word level in data-to-text generation.
\newblock \emph{arXiv preprint arXiv:2102.02810}.

\bibitem[{Rush et~al.(2015)Rush, Chopra, and Weston}]{rush2015neural}
Alexander~M. Rush, Sumit Chopra, and Jason Weston. 2015.
\newblock A neural attention model for abstractive sentence summarization.
\newblock In \emph{Proceedings of the 2015 Conference on Empirical Methods in
  Natural Language Processing}, pages 379--389, Lisbon, Portugal. Association
  for Computational Linguistics.

\bibitem[{Scao et~al.(2022)Scao, Fan, Akiki, Pavlick, Ili{\'c}, Hesslow,
  Castagn{\'e}, Luccioni, Yvon, Gall{\'e} et~al.}]{scao2022bloom}
Teven~Le Scao, Angela Fan, Christopher Akiki, Ellie Pavlick, Suzana Ili{\'c},
  Daniel Hesslow, Roman Castagn{\'e}, Alexandra~Sasha Luccioni, Fran{\c{c}}ois
  Yvon, Matthias Gall{\'e}, et~al. 2022.
\newblock Bloom: A 176b-parameter open-access multilingual language model.
\newblock \emph{arXiv preprint arXiv:2211.05100}.

\bibitem[{See et~al.(2017)See, Liu, and Manning}]{see2017get}
Abigail See, Peter~J. Liu, and Christopher~D. Manning. 2017.
\newblock Get to the point: Summarization with pointer-generator networks.
\newblock In \emph{Proceedings of the 55th Annual Meeting of the Association
  for Computational Linguistics (Volume 1: Long Papers)}, pages 1073--1083,
  Vancouver, Canada. Association for Computational Linguistics.

\bibitem[{Sennrich et~al.(2016)Sennrich, Haddow, and
  Birch}]{sennrich2016controlling}
Rico Sennrich, Barry Haddow, and Alexandra Birch. 2016.
\newblock Controlling politeness in neural machine translation via side
  constraints.
\newblock In \emph{Proceedings of the 2016 Conference of the North {A}merican
  Chapter of the Association for Computational Linguistics: Human Language
  Technologies}, pages 35--40, San Diego, California. Association for
  Computational Linguistics.

\bibitem[{Shimorina et~al.(2019)Shimorina, Khasanova, and
  Gardent}]{shimorina2019creating}
Anastasia Shimorina, Elena Khasanova, and Claire Gardent. 2019.
\newblock Creating a corpus for {R}ussian data-to-text generation using neural
  machine translation and post-editing.
\newblock In \emph{Proceedings of the 7th Workshop on Balto-Slavic Natural
  Language Processing}, pages 44--49, Florence, Italy. Association for
  Computational Linguistics.

\bibitem[{Sutskever et~al.(2011)Sutskever, Martens, and
  Hinton}]{sutskever2011generating}
Ilya Sutskever, James Martens, and Geoffrey~E. Hinton. 2011.
\newblock Generating text with recurrent neural networks.
\newblock In \emph{Proceedings of the 28th International Conference on Machine
  Learning, {ICML} 2011, Bellevue, Washington, USA, June 28 - July 2, 2011},
  pages 1017--1024. Omnipress.

\bibitem[{Sutskever et~al.(2014)Sutskever, Vinyals, and
  Le}]{sutskever2014sequence}
Ilya Sutskever, Oriol Vinyals, and Quoc~V. Le. 2014.
\newblock Sequence to sequence learning with neural networks.
\newblock In \emph{Advances in Neural Information Processing Systems 27: Annual
  Conference on Neural Information Processing Systems 2014, December 8-13 2014,
  Montreal, Quebec, Canada}, pages 3104--3112.

\bibitem[{Taylor et~al.(2022)Taylor, Kardas, Cucurull, Scialom, Hartshorn,
  Saravia, Poulton, Kerkez, and Stojnic}]{taylor2022galactica}
Ross Taylor, Marcin Kardas, Guillem Cucurull, Thomas Scialom, Anthony
  Hartshorn, Elvis Saravia, Andrew Poulton, Viktor Kerkez, and Robert Stojnic.
  2022.
\newblock Galactica: A large language model for science.
\newblock \emph{arXiv preprint arXiv:2211.09085}.

\bibitem[{Tian et~al.(2019)Tian, Narayan, Sellam, and
  Parikh}]{tian2019sticking}
Ran Tian, Shashi Narayan, Thibault Sellam, and Ankur~P Parikh. 2019.
\newblock Sticking to the facts: Confident decoding for faithful data-to-text
  generation.
\newblock \emph{arXiv preprint arXiv:1910.08684}.

\bibitem[{Touvron et~al.(2023)Touvron, Lavril, Izacard, Martinet, Lachaux,
  Lacroix, Rozi{\`e}re, Goyal, Hambro, Azhar et~al.}]{touvron2023llama}
Hugo Touvron, Thibaut Lavril, Gautier Izacard, Xavier Martinet, Marie-Anne
  Lachaux, Timoth{\'e}e Lacroix, Baptiste Rozi{\`e}re, Naman Goyal, Eric
  Hambro, Faisal Azhar, et~al. 2023.
\newblock Llama: Open and efficient foundation language models.
\newblock \emph{arXiv preprint arXiv:2302.13971}.

\bibitem[{Vaswani et~al.(2018)Vaswani, Bengio, Brevdo, Chollet, Gomez, Gouws,
  Jones, Kaiser, Kalchbrenner, Parmar, Sepassi, Shazeer, and
  Uszkoreit}]{vaswani2018tensor2tensor}
Ashish Vaswani, Samy Bengio, Eugene Brevdo, Francois Chollet, Aidan Gomez,
  Stephan Gouws, Llion Jones, {\L}ukasz Kaiser, Nal Kalchbrenner, Niki Parmar,
  Ryan Sepassi, Noam Shazeer, and Jakob Uszkoreit. 2018.
\newblock {T}ensor2{T}ensor for neural machine translation.
\newblock In \emph{Proceedings of the 13th Conference of the Association for
  Machine Translation in the {A}mericas (Volume 1: Research Track)}, pages
  193--199, Boston, MA. Association for Machine Translation in the Americas.

\bibitem[{Vaswani et~al.(2017)Vaswani, Shazeer, Parmar, Uszkoreit, Jones,
  Gomez, Kaiser, and Polosukhin}]{vaswani2017attention}
Ashish Vaswani, Noam Shazeer, Niki Parmar, Jakob Uszkoreit, Llion Jones,
  Aidan~N. Gomez, Lukasz Kaiser, and Illia Polosukhin. 2017.
\newblock Attention is all you need.
\newblock In \emph{Advances in Neural Information Processing Systems 30: Annual
  Conference on Neural Information Processing Systems 2017, December 4-9, 2017,
  Long Beach, CA, {USA}}, pages 5998--6008.

\bibitem[{Vedantam et~al.(2015)Vedantam, Zitnick, and
  Parikh}]{vedantam2015cider}
Ramakrishna Vedantam, C.~Lawrence Zitnick, and Devi Parikh. 2015.
\newblock Cider: Consensus-based image description evaluation.
\newblock In \emph{{IEEE} Conference on Computer Vision and Pattern
  Recognition, {CVPR} 2015, Boston, MA, USA, June 7-12, 2015}, pages
  4566--4575. {IEEE} Computer Society.

\bibitem[{Vig and Belinkov(2019)}]{vig2019analyzing}
Jesse Vig and Yonatan Belinkov. 2019.
\newblock Analyzing the structure of attention in a transformer language model.
\newblock In \emph{Proceedings of the 2019 ACL Workshop BlackboxNLP: Analyzing
  and Interpreting Neural Networks for NLP}, pages 63--76, Florence, Italy.
  Association for Computational Linguistics.

\bibitem[{Wu et~al.(2016)Wu, Schuster, Chen, Le, Norouzi, Macherey, Krikun,
  Cao, Gao, Macherey et~al.}]{wu2016google}
Yonghui Wu, Mike Schuster, Zhifeng Chen, Quoc~V Le, Mohammad Norouzi, Wolfgang
  Macherey, Maxim Krikun, Yuan Cao, Qin Gao, Klaus Macherey, et~al. 2016.
\newblock Google's neural machine translation system: Bridging the gap between
  human and machine translation.
\newblock \emph{arXiv preprint arXiv:1609.08144}.

\bibitem[{Xia et~al.(2019)Xia, He, Tan, Tian, He, and Qin}]{xia2019tied}
Yingce Xia, Tianyu He, Xu~Tan, Fei Tian, Di~He, and Tao Qin. 2019.
\newblock Tied transformers: Neural machine translation with shared encoder and
  decoder.
\newblock In \emph{Proceedings of the AAAI Conference on Artificial
  Intelligence}, volume~33, pages 5466--5473.

\bibitem[{Zhang et~al.(2022)Zhang, Roller, Goyal, Artetxe, Chen, Chen, Dewan,
  Diab, Li, Lin et~al.}]{zhang2022opt}
Susan Zhang, Stephen Roller, Naman Goyal, Mikel Artetxe, Moya Chen, Shuohui
  Chen, Christopher Dewan, Mona Diab, Xian Li, Xi~Victoria Lin, et~al. 2022.
\newblock Opt: Open pre-trained transformer language models.
\newblock \emph{arXiv preprint arXiv:2205.01068}.

\bibitem[{Zhu et~al.(2020)Zhu, Xia, Wu, He, Qin, Zhou, Li, and
  Liu}]{zhu2020incorporating}
Jinhua Zhu, Yingce Xia, Lijun Wu, Di~He, Tao Qin, Wengang Zhou, Houqiang Li,
  and Tie{-}Yan Liu. 2020.
\newblock Incorporating {BERT} into neural machine translation.
\newblock In \emph{8th International Conference on Learning Representations,
  {ICLR} 2020, Addis Ababa, Ethiopia, April 26-30, 2020}. OpenReview.net.

\end{thebibliography}
\bibliographystyle{acl_natbib}

%%%%%%%%%%%%%%%%%%%%%%%%%%%%%%%%%%%%%%%%%%%%%%%%%%%%%%%%%%%%

\clearpage
\appendix
\renewcommand{\thesection}{A.\arabic{section}}
\setcounter{theorem}{0}
\setcounter{lemma}{0}
\setcounter{corollary}{0}
\setcounter{proposition}{0}

\begin{center}
    {\LARGE \textbf{Appendix. Supplementary Material}}
    \end{center}

  \section{Proof of Proposition \ref{prop:pertube}}\label{proof:pertube}
  \begin{proposition}{ \ref{prop:pertube}}
    Given a function $y=f(x)$ with a Jacobian matrix $J_f$, if we have a pertubation vector $\Delta x$ and $y+\Delta y=f(x + \Delta x)$, then
    \begin{equation}\small \frac{\|\Delta y\|}{\|\Delta x\|} \le \|J_f\| + o(1) .
    \end{equation}
  \end{proposition}
  \begin{proof}
  Let $y=f(x)$ be given, in which $x,y\in \mathbb{R}^d$. We denote the output of $f$ when $x$ is perturbed by a vector $\Delta x$ as $f(x+\Delta x)$. Since for the Jacobian matrix, we have $f(x+\Delta x)-f(x)=J_f \Delta x + o(\|\Delta x\|) $, it follows that:
  $$\begin{aligned} 
  f(x+\Delta x) - f(x)&=J_f \Delta x + o(\|\Delta x\|)\\
  \|f(x+\Delta x) - f(x)\|&=\|J_f \Delta x + o(\|\Delta x\|)\|\\
  \|f(x+\Delta x) - f(x)\| &\le \|J_f\Delta x\| + \|o(\|\Delta x\|)\| \\
  \|f(x+\Delta x) - f(x)\| &\le \|J_f\| \|\Delta x\|  + \|o(\|\Delta x\|)\|\\
  \|\Delta y\| &\le \|J_f\| \|\Delta x\|  + \|o(\|\Delta x\|)\| \\
  \frac{\|\Delta y\|}{\|\Delta x\|} & \le \|J_f\| + \frac{o(\|\Delta x\|)}{\|\Delta x\|}\\
  \frac{\|\Delta y\|}{\|\Delta x\|} & \le \|J_f\| + o(1)
  \end{aligned}$$
  \end{proof}
  
  \section{Proof of Theorem \ref{thm:jacobs}}\label{proof:proof-thm-bound}

  \begin{lemma}\label{lemma:KQQ}
    Let $Z=\mathtt{Softmax}(YA^\top X^\top )XW$, $z_i^\top $ be the $i$th row of $Z$ and $x_j^\top $ be the $j$th row of $X$. Then,
    $$\begin{aligned}
    J_{ij}=&\frac{\partial z_i}{\partial x_j}\\
    =&W^\top (X^\top  (\operatorname{Diag}(p_i)-p_ip_i^\top )\cdot  (e_{ji}YA^\top )+Ip_{ij})
    %=&Y^\top (\operatorname{Diag}(p_i)-p_ip_i^\top )\cdot \\ &(e_{ji}YA^\top +YA\delta_{ij})+Ip_{ij},
    \end{aligned}$$
  in which $P=\mathtt{Softmax}(YA^\top X^\top )$, $p_i^\top $ is the $i$th row of $P$, $e_{ji}$ is a binary matrix with zeros exerywhere except the $(i, j)$th entry, $N$ is the row number of $X$.
  \end{lemma}

  \begin{proof}
    $$z_i=(PXW)_i^\top =\sum_{k=1}^NW^\top x_kp_{ik}$$
    
    $$\begin{aligned}
      J_{ij}&=\frac{\partial \sum_{k=1}^NW^\top x_kp_{ik}}{\partial x_j}\\
      &=\sum_{k=1}^N W^\top x_k\frac{\partial p_{ik}}{\partial x_j}+\sum_{k=1}^N \frac{\partial W^\top x_k}{\partial x_j}p_{ik}\\
      &=W^\top [x_1,x_2,\cdots,x_N]\begin{bmatrix}
          \frac{\partial p_{i1}}{\partial z_j}\\
          \frac{\partial p_{i2}}{\partial z_j}\\
          \vdots\\
          \frac{\partial p_{iN}}{\partial z_j}
      \end{bmatrix}+\frac{\partial W^\top x_j}{\partial x_j}p_{ij}\\
    \end{aligned}$$
    $$\begin{aligned}
      &=W^\top X^\top \frac{\partial p_{i}}{\partial x_j}+W^\top p_{ij}\\
      &=W^\top (X^\top \frac{\partial \mathtt{Softmax}((YA^\top X^\top )_i^\top )}{\partial x_j}+Ip_{ij})\\
      &=W^\top (X^\top  \mathtt{Softmax}'((YA^\top X^\top )_i^\top )) \frac{\partial (XAy_i)}{\partial x_j}+Ip_{ij})\\
      &=W^\top (X^\top  \mathtt{Softmax}'((YA^\top X^\top )_i^\top ))\cdot \\& \quad \frac{\partial ([x_1,x_2,\cdots,x_N]^\top Ay_i)}{\partial x_j}+Ip_{ij})\\
      &=W^\top (X^\top  (\operatorname{Diag}(p_i)-p_ip_i^\top )\cdot \\ &\quad ([\frac{\partial x_1^\top Ay_i}{\partial x_j},\frac{\partial x_2^\top Ay_i}{\partial x_j},\cdots,\frac{\partial x_N^\top Ay_i}{\partial x_j}]^\top )+Ip_{ij})\\
      &=W^\top (X^\top  (\operatorname{Diag}(p_i)-p_ip_i^\top )\cdot \\ &\quad ([0,\cdots,\underbrace{Ay_i}_{j\text{th } col},\cdots,0]^\top )+Ip_{ij})\\
      &=W^\top (X^\top  (\operatorname{Diag}(p_i)-p_ip_i^\top )\cdot  (e_{ji}YA^\top )+Ip_{ij})\\
    \end{aligned}$$
    \end{proof}

  \begin{lemma}\label{lemma:QQQ}
    Let $Z=\mathtt{Softmax}(QA^\top Q^\top )QW$, $z_i^\top $ be the $i$th row of $Z$ and $q_j^\top $ be the $j$th row of $Q$.
    $$\begin{aligned}
    J_{ij}=&\frac{\partial z_i}{\partial q_j}\\
    =&W^\top (Q^\top (\operatorname{Diag}(p_i)-p_ip_i^\top )\cdot (e_{ji}QA^\top +QA\delta_{ij})+Ip_{ij}),\end{aligned}$$
  in which $P=\mathtt{Softmax}(QA^\top Q^\top )$, $\delta_{ij}$ is a scalar equals to 1 if $i=j$ and equals to 0 otherwise.
  \end{lemma}
  The detailed proof can be found in \citet{kim2020lipschitz}. To prove Theorem \ref{thm:jacobs}, we follow the obsevation of \citet{kim2020lipschitz} that the Softmax matrix $P$ is a stochastic matrix, namely, its entries are non-negative and its rows sum to 1. For each element $p_{ij}$ in $P$, $p_{ij}\in[0,1]$ and they have an equal chance of receiving attention. Therefore, we have $\mathbb{E}(p_{ij})=\frac{1}{d_X}$.

  \begin{theorem}{\ref{thm:jacobs}}
    For $Z^E=\mathtt{ATT}(Y,X,X)$, where $\|X\|, \|Y\|, \|A\|, \|W_V\|$ are bounded, $\exists$ $C_3\ge 0,\delta \in(0,1)$, with probability at least $1-\delta^2$,
    $$\|J_{ij}^E\|\le C_3 (\frac{1}{N}+\sqrt{\ln \frac{1}{\delta}}). $$ %(C_3 +  \sqrt{\frac{1}{N}+\sqrt{\frac{1}{N}\ln \frac{1}{\xi}}}).$$
  
  For $Z^C=\mathtt{ATT}(Y,[X^\top ,Y^\top ],[X^\top ,Y^\top ])$, with probability at least $1-\delta^2$,
  $$\|J_{ij}^C\|\le  C_3 (\frac{1}{N+i}+\sqrt{\ln \frac{1}{\delta}}).$$%\cdot \\ & (C_3 +  \sqrt{\frac{1}{N+i}+\sqrt{\frac{1}{N+i}\ln \frac{1}{\xi}}}).\end{aligned}$$
  \end{theorem}

  \begin{proof}
  We first prove the upper bound for $J_{ij}^E$. From Lemma \ref{lemma:KQQ}, we have:
  
  $$\begin{aligned}
    &J_{ij}^E=W^\top (X^\top  (\operatorname{Diag}(p_i)-p_ip_i^\top )\cdot (e_{ji}YA^\top )+Ip_{ij})\\
    &\|J_{ij}^E\|=\|W^\top (X^\top  (\operatorname{Diag}(p_i)-p_ip_i^\top ) e_{ji}YA^\top +Ip_{ij})\|\\
    &\le \|W^\top \|\|X^\top  (\operatorname{Diag}(p_i)-p_ip_i^\top ) e_{ji}YA^\top \| +\|Ip_{ij}\|\\
    &\le \|W^\top \|
    (\|X^\top \| \| (\operatorname{Diag}(p_i)-p_ip_i^\top ) e_{ji}\|\cdot \|YA^\top  \|+\|Ip_{ij}\|)\\
    &\le (\|X^\top \| \| \begin{bmatrix}&0 &\cdots &-p_{i1}p_{ij} &\cdots &0 \\&0 &\cdots &-p_{i2}p_{ij} &\cdots &0\\ & \vdots & \vdots& \vdots& \vdots& \vdots \\ &0 &\cdots &p_{ij}-p_{ij}p_{ij} &\cdots &0 \\ & \vdots & \vdots& \vdots& \vdots& \vdots \\ &0 &\cdots &-p_{iN}p_{ij} &\cdots &0 \end{bmatrix}  \|\cdot \|YA^\top  \|+  \|IP_{ij}\|)\|W^\top \|\\
    &= (\|X^\top \| \|p_{ij}\| \|[p_{i1}, p_{i2},\cdots, 1-p_{ij},\cdots,p_{iN}]^\top \|\cdot  \|YA^\top  \|+ \|Ip_{ij}\|)\|W^\top \|\\
    &\le C_1 p_{ij} \sqrt{p_{i1}^2+ p_{i2}^2+\cdots 1+p_{ij}^2-2p_{ij}+\cdots p_{iN}^2}  + C_2 p_{ij}\\
    &\le C_1 p_{ij} \sqrt{1+\|p_i\|^2} + C_2p_{ij}\\
    &\le C_1 p_{ij} \sqrt{1+\|p_i\|^2_1} + C_2p_{ij}\\
    &= C_1 \sqrt{2} p_{ij}+C_2 p_{ij}\\
    &= C_3 p_{ij},
  \end{aligned}$$
  where $C_1\ge 0,C_2\ge 1,C_3\ge 1$. Since we assume $p_{ij}\in[0,1]$ and $\mathbb{E}(p_{ij})=\frac{1}{N}$, Therefore, with Hoeffding inequality, for $p_{ij}$ we have:
  $$\begin{aligned}
  \operatorname{Pr}(p_{ij} \ge \frac{1}{N}+t) &\leq e^{-2t^2}\\
  \operatorname{Pr}(p_{ij} \ge \frac{1}{N}+\sqrt{\ln \frac{1}{\delta} }) &\leq \delta^2\\
  \operatorname{Pr}(p_{ij} \le \frac{1}{N}+\sqrt{\ln \frac{1}{\delta}}) &\leq 1-\delta^2.
  \end{aligned}$$
  
  %For $p_i$, we have:
  %$$\begin{aligned}
  %  \operatorname{Pr}(\|p_{i}\|^2 \ge \frac{1}{N}+t) \leq e^{-2Nt^2}\\
  %  \operatorname{Pr}(\|p_{i}\|^2 \ge \frac{1}{N}+\sqrt{\frac{1}{N}\ln \frac{1}{\xi}}) &\leq \xi^2\\
  %  \operatorname{Pr}(\|p_{i}\| \le \sqrt{\frac{1}{N}+\sqrt{\frac{1}{N}\ln \frac{1}{\xi}}}) &\geq 1-\xi^2\\
  %\end{aligned}$$
  
  Therefore,
  with probability at least $1-\delta^2$
  $$\|J_{ij}\|\le C_3 (\frac{1}{N}+\sqrt{\ln \frac{1}{\delta}})$$ %(C_3 +  \sqrt{\frac{1}{N}+\sqrt{\frac{1}{N}\ln \frac{1}{\xi}}})$$

  Next, we prove the upper bound for $J_{ij}^C$. We denote $Q$ as
  $Q=\begin{bmatrix}
    X\\Y
  \end{bmatrix}$. Then, $J_{ij}^C$ is the Jacobian matrix for $\partial q_{N+i}/\partial q_j$.   From Lemma \ref{lemma:QQQ}, we have: 
  
  $$\begin{aligned}
    J_{ij}^C&=W^\top (Q^\top (\operatorname{Diag}(p_{N+i})-p_{N+i}p_{N+i}^\top )\cdot \\ &  \ \ \ \ \ \  (e_{j,{N+i}}QA^\top +QA\delta_{{N+i},j})+Ip_{{N+i},j})\\
    \|J_{ij}^C\|&=\|W^\top (Q^\top (\operatorname{Diag}(p_{N+i})-p_{N+i}p_{N+i}^\top )\cdot \\ &  \ \ \ \ \ \  (e_{j,{N+i}}QA^\top +QA\delta_{{N+i},j})+Ip_{{N+i},j})\|\\
    &=\|W^\top (Q^\top (\operatorname{Diag}(p_{N+i})-p_{N+i}p_{N+i}^\top )\cdot \\ &  \ \ \ \ \ \  (e_{j,{N+i}}QA^\top )+Ip_{{N+i},j})\|\\
    &\le C_3 p_{{N+i},j},
  \end{aligned}$$
  
  where $C_1\ge 0,C_3\ge 1$. Since we assume $p_{{N+i},j}\in[0,1]$ and $\mathbb{E}(p_{{N+i},j})=\frac{1}{N+i}$, Therefore, with Hoeffding inequality, for $p_{{N+i},j}$ we have:
  $$\begin{aligned}
  \operatorname{Pr}(p_{{N+i},j} \ge \frac{1}{N+i}+t) &\leq e^{-2t^2}\\
  \operatorname{Pr}(p_{{N+i},j} \ge \frac{1}{N+i}+\sqrt{\ln \frac{1}{\delta} }) &\leq \delta^2\\
  \operatorname{Pr}(p_{{N+i},j} \le \frac{1}{N+i}+\sqrt{\ln \frac{1}{\delta}}) &\leq 1-\delta^2.
  \end{aligned}$$
  
  %For $p_{N+i}$, we have:
  %$$\begin{aligned}
  %  \operatorname{Pr}(\|p_{N+i}\|^2 \ge \frac{1}{N+i}+t) \leq e^{-2(N+i)t^2}\\
  %  \operatorname{Pr}(\|p_{N+i}\|^2 \ge \frac{1}{N+i}+\sqrt{\frac{1}{N+i}\ln \frac{1}{\xi}}) &\leq \xi^2\\
  %  \operatorname{Pr}(\|p_{N+i}\| \le \sqrt{\frac{1}{N+i}+\sqrt{\frac{1}{N+i}\ln \frac{1}{\xi}}}) &\geq 1-\xi^2\\
  %\end{aligned}$$
  
  Therefore,
  with probability at least $1-\delta^2$
  $$\|J_{ij}^C\|\le  C_3 (\frac{1}{N+i}+\sqrt{\ln \frac{1}{\delta}}) $$ %\\ & (C_3 +  \sqrt{\frac{1}{N+i}+\sqrt{\frac{1}{N+i}\ln \frac{1}{\xi}}})\end{aligned}$$
  
  \end{proof}
  
  \section{Detailed Model Structure}\label{sec:structure}
  In this section, we provide a more detailed description of each model including LM (\Cref{tab:lm}), ED (\Cref{tab:ed}), RED (\Cref{tab:red}) and PALM (\Cref{tab:palm}). We denote $\mathcal{F}_l$ as the feedforward layer in the $l$th block. The positional embedding is denoted as $E_p=\operatorname{Emb}_p(p_s)$ while the word embedding layer is denoted as $E_w=\operatorname{Emb}_w(s)$. $E_w,E_p\in \mathbb{R}^{|s|\times d}$ and $d$ is the dimension size.
  
  \begin{table}[h]
    \centering
    \scriptsize
    
    \begin{tabular}{@{~}l@{~}}
      \toprule
      \multicolumn{1}{c}{Language Model}\\
      \midrule
      {$a=[s_1,s_2,\cdots, s_{|s|},\Rightarrow,t_1,t_2,\cdots, t_{|t|}]$} \\ 
      {$p_a=[1,2,\cdots,|s|+|t|]$} \\ 
      $E_p=\operatorname{Emb}_p(p_a) ; E_w=\operatorname{Emb}_w(a) $ \\ 
      $G_1=E_p+E_w$ \\ 
      {for $l$ in [1, $\cdots$, 6]:} \\ 
      $\quad Q_l=\mathtt{ATT}_l(G_l,G_l,G_l)$ \\
      $\quad G_{l+1}=H_l=\mathcal{F}_l(Q_l)$\\ 
      \hline
      $L=\mathcal{L}(H_{-1},a)$ \\ 
      \bottomrule
    \end{tabular}
    \caption{Language model framework.}
    
  \label{tab:lm}
  %\vspace{-2em}
  \end{table}

  \begin{table}[h]
    \centering
    \scriptsize
    %\vspace{0pt}
    \begin{tabular}{@{~}l@{~}|@{~}l@{~}}
      \toprule
      \multicolumn{2}{c}{Regularized Encoder Decoder}  \\
      \hline
      \multicolumn{1}{c}{Encoder}& \multicolumn{1}{c}{Decoder}\\
      \hline
      {$p_s=[1,2,\cdots,|s|]$}& {$p_t=[|s|+1,\cdots,|s|+|t|+1]$}\\ 
      {$E_p^{E}=\operatorname{Emb}_p(p_s) $}& {$E_p^{D}=\operatorname{Emb}_p(p_t)$}\\ 
      {$E_w^{E}=\operatorname{Emb}_w(s)$}& {$E_w^{D}=\operatorname{Emb}_w(\text{`}\Rightarrow\text{'}+t)$}\\ 
      {$G_1^{E}=E_p^{E}+E_w^{E}$}& {$G_1^{D}=E_p^{D}+E_w^{D}$}\\ 
      {for $l$ in [1, $\cdots$, 6]:}& {for $l$ in [1, $\cdots$, 6]:}\\ 
      {\quad $Q_l^E=\mathtt{ATT}_l(G_l^{E},G_l^{E},G_l^{E})$}& {\quad $Q_l^D=\mathtt{ATT}_i(G_l^{D},\begin{bmatrix}G_l^{E}\\G_l^{D}\end{bmatrix},\begin{bmatrix}G_l^{E}\\G_l^{D}\end{bmatrix}$)}\\
      {\quad$G_{l+1}^{E}=H_{l}^E=\mathcal{F}_l(Q_l^E)$}& {\quad$G_{l+1}^{D}=H_{l}^D=\mathcal{F}_l(G_l^D)$}\\ 
      \hline
      \multicolumn{2}{c}{$L=\mathcal{L}(H_{-1}^E,s)+\mathcal{L}(H_{-1}^D,t)$}\\
      \bottomrule
    \end{tabular}
  \caption{Regularized encoder decoder framework.}
  \label{tab:red}
  %\vspace{-2em}
  \end{table}
  
  \begin{table}[h]
    \centering
    \scriptsize
      \centering
      %\vspace{0pt}
      \begin{tabular}{@{~}l@{~}|@{~}l@{~}}
        \toprule
        \multicolumn{2}{c}{Encoder-Decoder}  \\
        \hline
        \multicolumn{1}{c}{Encoder}& \multicolumn{1}{c}{Decoder}\\
        \hline
        {$p_s=[1,2,\cdots,|s|]$}& {$p_t=[1, 2,\cdots,|t|]$}\\ 
        {$E_p^{E}=\operatorname{Emb}_p^E(p_s) $}& {$E_p^{D}=\operatorname{Emb}_p^D(p_t)$}\\ 
        {$E_w^{E}=\operatorname{Emb}_w^E(s)$}& {$E_w^{D}=\operatorname{Emb}_w^D(t)$}\\ 
        {$G_1^{E}=E_p^{E}+E_w^{E}$}& {$G_1'^{D}=E_p^{D}+E_w^{D}$}\\ 
        {for $l$ in [1, $\cdots$, 6]:}& {for $l$ in [1, $\cdots$, 6]:}\\ 
        {\quad $Q_l^E=\mathtt{ATT}_l^E(G_l^{E},G_l^{E},G_l^{E})$}& {\quad $G_l^D=\mathtt{ATT}_l^D(G_l'^{D},G_l'^{D},G_l'^{D})$}\\
        {\quad$G_{l+1}^E=H_l^E=\mathcal{F}_l^E(G_l^E)$} & {\quad $Q_l^D=\mathtt{ATT}_i^J(G_l^D,G_l^E,G_l^E)$}\\ 
        {} & {\quad$G_{l+1}^D=H_{l+1}^D=\mathcal{F}_l^D(Q_l^D)$}\\
        \hline
        \multicolumn{2}{c}{$L=\mathcal{L}(H_{-1}^D,t)$}\\
        \bottomrule
      \end{tabular}
      \caption{Encoder-decoder framework.}
      \label{tab:ed}
  %    \vspace{-2em}
  \end{table}

  \begin{table}[!h]
    \centering
    \scriptsize
    
    \begin{tabular}{@{~}l@{~}}
      \toprule
      \multicolumn{1}{c}{Partial Attention Language Model}\\
      \midrule
      {$a=[s_1,s_2,\cdots, s_{|s|},\Rightarrow,t_1,t_2,\cdots, t_{|t|}]$} \\ 
      {$p_a=[1,2,\cdots,|s|,1,2,\cdots,|t|+1]$} \\ 
      {$l_a=[0,0,\cdots,0,1,1,\cdots,1]$} \\
      $E_p=\operatorname{Emb}_p(p_a) ; E_w=\operatorname{Emb}_w(a); E_l=\operatorname{Emb}_l(l_a) $ \\ 
      $G_1=E_p+E_w+E_l$ \\ 
      {for $l$ in [1, $\cdots$, 6]:} \\ 
      $\quad Q_l=\mathtt{ATT}_l(G_l,G_l,G_l)$ \\
      $\quad P_l=\mathcal{F}_{P,l}(Q_{l[1:|s|]});R_l=\mathtt{ATT}_l^P(Q_l,P_l,P_l)$\\
      $\quad G_{l+1}=H_{l}=\mathcal{F}_i(R_l)$\\ 
      \hline
      $L=\mathcal{L}(H_{-1},a)$ \\ 
      \bottomrule
    \end{tabular}
    \caption{Partial attention language model framework.}
    
  \label{tab:palm}
  %\vspace{-2em}
  \end{table}
  
  Here, for the feedforward layer $\mathcal{F}_{P,l}$ in \Cref{tab:palm}, it can be interpreted as:

  $$\begin{aligned}
    &P_{l1}=\operatorname{Dropout}(\tanh ( G_{l[1:|s|]}W_{P,l1} + \bold{1}b_{P,l1}^\top ))\\
    &P_{l2}=\operatorname{Dropout}( P_{l1}W_{P,l2} + \bold{1}b_{P,l2}^\top)\\
    &P_{l}= P_{l2} +P_{l1},
  \end{aligned}$$
  where $W_{P,l1}, W_{P,l2}\in \mathbb{R}^{d\times d}, b_{P,l1}, b_{P,l2} \in \mathbb{R}^d$ are trainable parameters, $\bold{1}\in \mathbb{R}^{|s|\times 1}$ is a vector with all elements equal to 1.

\end{document}